\documentclass[]{fairmeta}
\usepackage{pdflscape}
\usepackage{amsmath,amssymb,amsthm}
\usepackage{xspace}

\newtheorem{proposition}{Proposition}[section]

\newcommand{\oursb}{\mbox{RobustReview}\xspace}
\newcommand{\oursm}{\mbox{SciCore}\xspace}

\title{A Missing Piece for Trustworthy AI Reviewers: From Benchmarking Rhetorical Robustness to SciCore Review}

\author[*,1]{Chenguang Wang}
\author[*,2]{Ming Li}
\author[2]{Chenrui Fan}
\author[1]{Jianpeng Chen}
\author[3]{Han Chen}
\author[3]{Tianyi Zhou}
\author[1]{Dawei Zhou}

\renewcommand\affiliation[2][]{%
  \addtolist[#1]{#2}{\affiliationlist}{\affiliationformat}{~~~~}%
}

\affiliation[1]{Virginia Tech}
\affiliation[2]{University of Maryland}
\affiliation[3]{MBZUAI}

\contribution[*]{Co-first Author}

\authoremails{\email{minglii@umd.edu}, \email{\{cswang, dzhou\}@vt.edu}, \email{Tianyi.Zhou@mbzuai.ac.ae}}
\metadata[Project Page]{\url{https://github.com/c-steve-wang/Robust_Review}}

\abstract{
AI reviewers can assign different judgments to manuscripts that report the same science in different wording, potentially rewarding rhetorical optimization over scientific improvement. We formulate \textbf{Rhetorical Robustness} as the joint requirement of stability across content-preserving rewrites and discrimination across papers. We introduce \textbf{\oursb{}}, a controlled full-manuscript benchmark with 1,260 manuscript versions, and evaluate 30 reviewer configurations. The benchmark reveals \emph{false robustness}, where low rewrite sensitivity coincides with score collapse across papers, and shows that human alignment and rhetorical robustness rank reviewers differently. Moreover, the evaluated content-focused prompting protocol does not consistently improve robustness across backbones. Motivated by these findings, we introduce \textbf{\oursm{}}, a dual-branch reviewer that averages a full-manuscript judgment with a judgment based on an extracted, structured science core. This design combines manuscript-level assessment with a content-normalized view intended to reduce rhetorical sensitivity. In our primary GPT-5.5 comparison, \oursm{} achieves a leading joint stability-discrimination profile among the benchmarked reviewers while maintaining competitive human alignment. These results identify rhetorical robustness as a distinct evaluation target and demonstrate the potential of science-core review to improve it.

}

\begin{document}

\maketitle

\section{Introduction}

Large language models (LLMs) are increasingly being used to support scientific peer review, from generating manuscript feedback to assisting with review and decision making~\citep{wang2020reviewrobot,liang2024can,thakkar2026large,chen2026peercheck}. Because review judgments shape which work is accepted, revised, and disseminated, the reliability of AI reviewers matters not only for individual manuscripts but also for the broader scientific record~\citep{fytas2021makes,li2025developing}. Existing evaluations commonly ask whether AI-generated reviews are useful, resemble expert feedback, or reproduce human scores and decisions~\citep{liang2024can,zhou2024llm,li2025unveiling,chen2026peercheck}. These criteria are necessary, but they do not fully characterize a trustworthy AI reviewer. Such a reviewer should also preserve its scientific judgments when the same reported science is expressed in rhetorically different ways. We call this property \textbf{Rhetorical Robustness}, and argue that it is an important requirement for trustworthy AI reviewers. As LLMs make it increasingly easy to rewrite and strategically optimize manuscripts at low cost,\textbf{ review judgments that can be manipulated through wording alone would reward rhetorical optimization over scientific improvement and undermine the credibility of AI-based review}~\citep{kaneko2026paraphrasing,li2026gaming,li2026rhetoric}. Presentation may legitimately shape assessments of clarity and communicative quality, but it should not unduly alter judgments of scientific merit when the scientific content is preserved~\citep{james2024rigour}.

Existing work shows that AI review and LLM judging systems can be manipulated by overt instructions, adversarial phrasing, and strategically constructed text~\citep{ye2024we,lin2025breaking,collu2026misleading}. More importantly for scientific review, visible and meaning-preserving revisions to titles, abstracts, and full manuscripts can also alter automated evaluations to a large extent~\citep{du-2025-titletrap,kaneko2026paraphrasing,li2026gaming,baumann2026stop,yang2026no,li2026rhetoric}. However, the evaluation and design of AI reviewers have generally not treated rhetorical robustness as a joint requirement of stable judgment and scientific discrimination. Existing evaluations have extensively examined human alignment~\citep{liang2024can,chen2026peercheck}, while rhetorical robustness has remained a comparatively neglected dimension.

We therefore formulate rhetorical robustness through two complementary requirements: \textbf{within-paper stability} across rhetorical variants designed to preserve reported scientific content, and \textbf{between-paper discrimination}. Stability alone is insufficient: a reviewer assigning nearly identical scores to every paper would appear robust while failing to distinguish papers. The joint requirement asks whether reviewers can resist rhetorical variation without collapsing differences across papers. Human alignment captures a separate property, since agreement with human judgments on original manuscripts does not establish stability across their rhetorical variants. This motivates our central question: \textbf{Can AI reviewers remain stable under rhetorical variation while retaining paper-level discrimination?}

We operationalize this requirement through \textbf{\oursb{}}, a controlled full-manuscript benchmark built from 60 anonymized ICLR 2026 submissions, sampled equally from six mean human-review score intervals to cover different levels of human assessment. We retain each original manuscript and construct matched variants under 10 rhetorical conditions using two independent LLM systems, yielding 1,260 manuscripts in total. We evaluate 30 reviewer configurations spanning rubric-instructed LLMs, specialized scientific-review models, and agentic review systems under multiple review protocols. Our evaluation combines rhetorical stability and paper discrimination with agreement with human review scores.

The evaluated content-focused prompting protocol does not consistently improve robustness across backbones. We therefore introduce \textbf{\oursm{}}, a dual-branch framework. Because rhetorical robustness requires greater stability without sacrificing meaningful paper-level discrimination, \oursm{} uses a content-normalized scientific judgment to complement manuscript-level assessment. One branch reviews the complete manuscript. The other extracts a structured record of the manuscript's reported problem, claims, methods, assumptions, evidence, results, and limitations, and then reviews that record directly with an adapted protocol. We call this record a \emph{science core}. The final overall assessment is the mean of the two branch scores. The science-core branch is motivated by invariant representation theory~\citep{eaton1989group,dubois2021lossy}: the extracted record should vary little across rhetorical realizations of the same reported science while preserving distinctions across papers.
In our experiments, \oursm{} achieves a leading joint stability-discrimination profile while maintaining competitive human alignment. These results suggest that augmenting conventional review with a content-normalized scientific judgment can improve rhetorical robustness without discarding manuscript-level assessment.

Our contributions are threefold:
\begin{itemize}
\item We argue for \textbf{Rhetorical Robustness} as an important requirement for trustworthy AI reviewers and formulate it jointly through within-paper stability and between-paper discrimination, with human alignment evaluated as a distinct property.

\item We introduce \textbf{\oursb{}}, a controlled full-manuscript benchmark spanning rewrite strategies, reviewer families, and review protocols. It shows that content-focused review is nontrivial, exposes widespread configuration-dependent sensitivity, and identifies false robustness as a central measurement failure.

\item We introduce \textbf{\oursm{}}, a dual-branch framework that averages a conventional full-manuscript judgment with a content-normalized judgment from an extracted science core. The method improves the joint robustness profile while maintaining competitive human alignment in our primary evaluation.
\end{itemize}

\begin{figure}[!t]
  \centering
  \includegraphics[width=\linewidth]{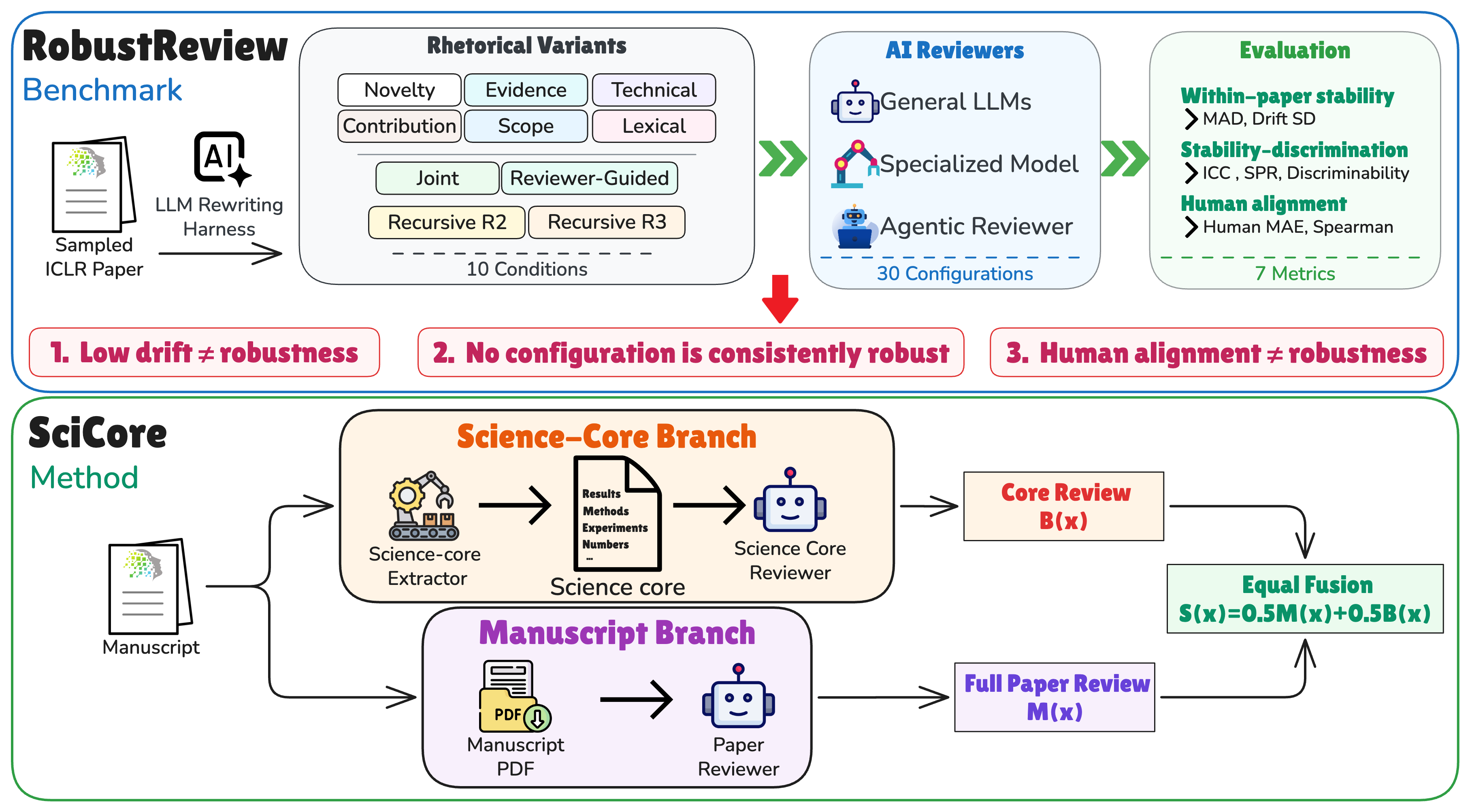}
  \caption{\textbf{Overview of the benchmark and method.} \oursb{} evaluates
  rhetorical robustness across controlled rhetorical variants, while \oursm{}
  averages a full-manuscript judgment with a content-normalized judgment
  obtained by extracting and reviewing the science core.}
  \label{fig:overview}
\end{figure}

\section{\oursb{}: Benchmark Design}
\label{sec:benchmark}
\label{sec:problem_formulation}

\subsection{Rhetorical Robustness}

We distinguish judgments of reported science from assessments of presentation. Clarity and style may legitimately affect the latter, but judgments of scientific contribution should not be unduly altered by rhetorical rewrites designed to preserve the reported scientific content. We therefore define \emph{rhetorical robustness} as the joint ability of an AI reviewer to maintain stable scientific judgments under such variation while retaining sensitivity to differences in reported scientific content across manuscripts.

Formally, for paper $i$, let $x_{i0}$ denote the original manuscript. Each controlled rewrite setting $k$, which specifies both a rhetorical condition and a rewrite producer, generates a variant $x_{ik}$ designed to preserve the paper's reported claims, methods, evidence, results, and conclusions while changing how this content is communicated. Such variation may involve claim stance, evidence framing, contribution organization, technical register, or lexical and syntactic realization. A reviewer configuration $m$ specifies both the reviewer model and review protocol. The resulting rewrite and review process is
\begin{equation}
\begin{aligned}
x_{ik} &= \operatorname{Rewrite}_{k}(x_{i0}), && k=1,\ldots,K,\\
y_{mik} &= \operatorname{Review}_{m}(x_{ik}), && k=0,\ldots,K.
\end{aligned}
\label{eq:rewrite_review_process}
\end{equation}
Here, $i=1,\ldots,N$ indexes papers, $x_{ik}$ is variant $k$ of paper $i$, and $y_{mik}$ is the judgment assigned by reviewer configuration $m$ to that presentation. A rhetorically robust reviewer should jointly satisfy two complementary requirements. First, \textbf{within-paper stability} requires judgments to remain consistent across rhetorical variants of the same paper. This condition alone is insufficient because indiscriminately constant scores would be maximally stable. Second, \textbf{between-paper discrimination} requires the reviewer to preserve distinctions based on the reported scientific content of different manuscripts, thereby ruling out this collapse. This requirement does not treat cross-paper score variation as ground-truth scientific merit; it asks whether paper differentiation remains large enough relative to rewrite-induced variation to be meaningful.

Rhetorical robustness therefore constitutes a \textbf{joint stability-discrimination requirement}: it requires within-paper stability without sacrificing between-paper discrimination. Section~\ref{sec:metrics} operationalizes this joint requirement with two direct within-paper metrics and three joint stability-discrimination metrics. The joint metrics do not measure between-paper behavior in isolation: each relates within-paper consistency to cross-paper variation or separation. Alignment with human reviewer judgments remains a distinct evaluation dimension, which we measure separately.

\subsection{Benchmark Construction}

\oursb{} operationalizes rhetorical robustness through \emph{matched paper families}, each containing an original manuscript and rhetorical variants designed to preserve its reported scientific content. Comparisons within each family directly measure rewrite-induced instability. Comparisons across families then provide the reference needed to determine whether that within-paper consistency coexists with differentiation among papers. Human judgments on the original manuscripts provide a separate reference for alignment.

We construct \oursb{} from 60 anonymized ICLR 2026 submissions with matched arXiv \LaTeX{} sources. We stratify eligible papers by their mean human overall-assessment rating and randomly sample 10 papers from each of six score intervals. This balanced sampling covers papers with different human-assessed ratings, while the human scores provide an external reference for evaluating the reviewers' assessments. We apply 10 rhetorical conditions. Six single-dimension conditions alter novelty stance, scope framing, evidence framing, contribution salience, technical register, or linguistic complexity. Four complex conditions apply a joint rewrite across dimensions, two or three recursive rewrite rounds (R2 and R3), or a reviewer-guided rewrite based on model feedback. Each condition is independently instantiated by GPT-5.5~\citep{openai_gpt55} and Claude Opus 4.8~\citep{anthropic_claude_opus48}, producing two variants per paper and condition. The resulting corpus contains 60 original manuscripts and 1,200 rhetorical variants, or 1,260 full manuscripts in total. All rewrites operate on complete \LaTeX{} projects under content-preservation and structural controls. Automated and human audits indicate that core technical content is largely preserved across the five assessed dimensions (Appendix~\ref{app:rewrite_fidelity}). Appendices~\ref{app:benchmark_inventory} and~\ref{app:rewrite_construction} give the sampling procedure, complete intervention definitions, rewrite procedure, and construction checks.

\subsection{Reviewer Configurations}
\label{sec:reviewers}

We evaluate 30 reviewer configurations across three system families. The general-purpose family comprises GPT-5.5~\citep{openai_gpt55}, GPT-5-mini~\citep{openai_gpt5mini}, Claude Sonnet 5~\citep{anthropic_claude_sonnet5}, GLM-5.2~\citep{glm5team2026glm5vibecodingagentic}, Kimi-K2.6~\citep{kimi_k2_6}, GPT-OSS-120B~\citep{openai2025gptoss120bgptoss20bmodel}, Gemini-3.5-Flash-Lite~\citep{google_gemini35flashlite}, and Qwen-3.5-Flash~\citep{qwen3.5flash}. Each is prompted under three protocols: \textsc{Standard}, \textsc{Strict}, and \textsc{Persistent}. \textsc{Standard} follows the ICLR review criteria and scoring scheme. \textsc{Strict} raises the evidentiary threshold through a more conservative rubric. \textsc{Persistent} retains the standard criteria while repeatedly instructing the reviewer to base scientific judgments on substantive content rather than rhetorical presentation. It therefore directly tests whether content-only instructions are sufficient to separate scientific judgment from rhetorical presentation. We additionally evaluate the specialized models OpenReviewer, CycleReviewer, and DeepReviewer~\citep{idahl-ahmadi-2025-openreviewer,weng2025cycleresearcherimprovingautomatedresearch,zhu-etal-2025-deepreview}, and the agentic systems AI Scientist, OpenJudge, and ProReviewer~\citep{lu2024aiscientistfullyautomated,openjudge2025,fang2026passivegenerationinvestigationproactive}, using their native review procedures. For every configuration, an original manuscript and all of its rhetorical variants are evaluated with the same procedure and scoring criteria. Appendix~\ref{app:reviewer_execution} reports the system configurations and execution settings.

\subsection{Evaluation Metrics}
\label{sec:metrics}

Following the definition above, we evaluate every reviewer configuration with seven metrics. MAD and Drift SD directly measure within-paper stability (lower is better). Because low drift alone can result from score collapse, ICC, SPR, and discriminability jointly evaluate within-paper consistency relative to cross-paper variation or separation (higher is better); we refer to these as \emph{joint stability-discrimination metrics}. Human MAE measures absolute agreement with mean human overall-assessment scores, and Spearman correlation measures agreement with the human ranking of the original papers (lower and higher are better, respectively). On the score-stratified benchmark, these human-alignment metrics complement the robustness metrics by assessing whether the reviewers' scores and rankings agree with human evaluations. Together, the seven metrics characterize stability under rhetorical rewriting, discrimination among papers, and alignment with human judgments. Appendix~\ref{app:metric_definitions} gives the complete definitions and equations.

\section{Benchmark Findings: Limits of Current AI Reviewers}
\label{sec:benchmark_results}

Table~\ref{tab:benchmark_main} reports all seven metrics for 30 existing reviewer configurations and \oursm{} under the same matched-manuscript design on \oursb{}. The within-paper metrics quantify absolute score movement, the joint metrics test whether stability coexists with paper-level discrimination, and the human-alignment metrics provide a separate external comparison.  The analysis here focuses on the existing reviewers, with \oursm{} included as a common-scale reference. Because the metrics capture different behaviors, no single column is sufficient for identifying a robust reviewer.

\begin{table}[!t]
  \centering
  \caption{\textbf{Main comparison of rhetorical robustness and human alignment.} MAD and Drift SD directly measure rewrite-induced within-paper instability. ICC, SPR, and discriminability are joint stability-discrimination metrics: each evaluates within-paper consistency relative to cross-paper variation or separation and should not be interpreted as a between-paper-only measure. Arrows indicate the preferred direction. The best point estimate in each column is shown in bold, the second-best is underlined, and the third-best is italicized.}
  \label{tab:benchmark_main}
  \begingroup
  \resizebox{\linewidth}{!}{%
  \begin{tabular}{@{}ll|cc|cc|ccc@{}}
    \toprule
    \multirow{2}{*}{\textbf{System}} & \multirow{2}{*}{\textbf{Protocol}} &
    \multicolumn{2}{c}{\textbf{Human alignment}} &
    \multicolumn{2}{c}{\textbf{Within-paper stability}} &
    \multicolumn{3}{c}{\textbf{Joint stability-discrimination}} \\
    \cmidrule(lr){3-4}\cmidrule(lr){5-6}\cmidrule(lr){7-9}
    & & H-MAE $\downarrow$ & Spearman $\uparrow$ &
    MAD $\downarrow$ & Drift SD $\downarrow$ & ICC $\uparrow$ &
    SPR $\uparrow$ & Discrim. $\uparrow$ \\
    \midrule
    \multicolumn{9}{@{}l}{\itshape General-purpose LLM reviewers} \\
    \cmidrule(lr){1-9}
    \multirow[t]{3}{*}{GPT-5.5~\citep{openai_gpt55}}
      & Standard   & 1.294 & 0.448 & 0.598 & 0.956 & 0.615 & \textit{0.555} & 0.649 \\
      & Strict     & \underline{1.078} & \textbf{0.529} & 0.766 & 1.202 & \textit{0.632} & 0.507 & \textit{0.672} \\
      & Persistent & 1.261 & 0.423 & 0.476 & 0.875 & \underline{0.697} & \underline{0.586} & \underline{0.682} \\
    \cmidrule(lr){1-9}
    \multirow[t]{3}{*}{GPT-5-mini~\citep{openai_gpt5mini}}
      & Standard   & 1.711 & 0.417 & 0.543 & 0.945 & 0.401 & 0.390 & 0.586 \\
      & Strict     & 1.240 & 0.403 & 0.895 & 1.221 & 0.393 & 0.424 & 0.598 \\
      & Persistent & 1.628 & 0.335 & 0.574 & 0.982 & 0.476 & 0.475 & 0.589 \\
    \cmidrule(lr){1-9}
    \multirow[t]{3}{*}{Claude Sonnet 5~\citep{anthropic_claude_sonnet5}}
      & Standard   & 1.189 & 0.414 & 0.468 & 0.863 & 0.579 & \textit{0.555} & 0.646 \\
      & Strict     & \textit{1.151} & 0.475 & 0.612 & 1.050 & 0.546 & 0.550 & 0.637 \\
      & Persistent & 1.239 & 0.360 & \underline{0.307} & 0.710 & 0.489 & 0.491 & 0.615 \\
    \cmidrule(lr){1-9}
    \multirow[t]{3}{*}{GLM-5.2~\citep{glm5team2026glm5vibecodingagentic}}
      & Standard   & 2.044 & $-0.034$ & 1.500 & 2.341 & 0.226 & 0.322 & 0.558 \\
      & Strict     & 2.047 & $-0.038$ & 1.548 & 2.094 & 0.212 & 0.389 & 0.556 \\
      & Persistent & 2.403 & $-0.190$ & 1.932 & 2.682 & 0.232 & 0.381 & 0.558 \\
    \cmidrule(lr){1-9}
    \multirow[t]{3}{*}{Kimi-K2.6~\citep{kimi_k2_6}}
      & Standard   & 1.689 & 0.280 & 1.353 & 1.995 & 0.211 & 0.428 & 0.555 \\
      & Strict     & 1.737 & 0.056 & 0.753 & 1.261 & 0.287 & 0.377 & 0.544 \\
      & Persistent & 1.578 & 0.050 & 1.139 & 1.645 & 0.283 & 0.392 & 0.568 \\
    \cmidrule(lr){1-9}
    \multirow[t]{3}{*}{GPT-OSS-120B~\citep{openai2025gptoss120bgptoss20bmodel}}
      & Standard   & 1.544 & 0.072 & 0.717 & 1.129 & 0.138 & 0.384 & 0.532 \\
      & Strict     & 1.692 & 0.119 & 0.385 & 0.888 & 0.091 & 0.353 & 0.508 \\
      & Persistent & 1.528 & $-0.008$ & 0.858 & 1.317 & 0.161 & 0.351 & 0.536 \\
    \cmidrule(lr){1-9}
    \multirow[t]{3}{*}{Gemini 3.5 Flash-Lite~\citep{google_gemini35flashlite}}
      & Standard   & 3.178 & 0.408 & \textbf{0.205} & \underline{0.631} & 0.199 & 0.468 & 0.511 \\
      & Strict     & 1.878 & 0.465 & 0.840 & 1.193 & 0.424 & 0.520 & 0.599 \\
      & Persistent & 2.828 & \textit{0.487} & 0.443 & 0.968 & 0.347 & 0.502 & 0.539 \\
    \cmidrule(lr){1-9}
    \multirow[t]{3}{*}{Qwen 3.5 Flash~\citep{qwen3.5flash}}
      & Standard   & 2.350 & 0.379 & 0.809 & 1.339 & 0.291 & 0.472 & 0.561 \\
      & Strict     & 1.239 & 0.409 & 0.866 & 1.340 & 0.226 & 0.404 & 0.541 \\
      & Persistent & 2.037 & 0.461 & 0.863 & 1.281 & 0.597 & 0.519 & 0.597 \\
    \midrule
    \multicolumn{9}{@{}l}{\itshape Specialized review models} \\
    \cmidrule(lr){1-9}
      OpenReviewer~\citep{idahl-ahmadi-2025-openreviewer}  & \multicolumn{1}{c}{--} & 1.406 & 0.339 & 1.080 & 1.524 & 0.215 & 0.374 & 0.542 \\
      CycleReviewer~\citep{weng2025cycleresearcherimprovingautomatedresearch} & \multicolumn{1}{c}{--} & 1.403 & 0.074 & 0.779 & 1.047 & 0.128 & 0.350 & 0.533 \\
      DeepReviewer~\citep{zhu-etal-2025-deepreview}  & \multicolumn{1}{c}{--} & 1.471 & 0.406 & 0.454 & 0.689 & 0.161 & 0.381 & 0.537 \\
    \midrule
    \multicolumn{9}{@{}l}{\itshape Agentic review systems} \\
    \cmidrule(lr){1-9}
      AI Scientist~\citep{lu2024aiscientistfullyautomated} & \multicolumn{1}{c}{--} & 1.513 & 0.094 & 0.838 & 1.253 & 0.213 & 0.382 & 0.544 \\
      OpenJudge~\citep{openjudge2025}    & \multicolumn{1}{c}{--} & 1.306 & 0.120 & \textit{0.335} & \textbf{0.597} & 0.239 & 0.444 & 0.526 \\
      ProReviewer~\citep{fang2026passivegenerationinvestigationproactive}  & \multicolumn{1}{c}{--} & 1.478 & 0.230 & 0.709 & 0.978 & 0.094 & 0.342 & 0.516 \\
    \midrule
    \multicolumn{9}{@{}l}{\itshape Our method} \\
    \cmidrule(lr){1-9}
      \textbf{\oursm{}} & \multicolumn{1}{c}{--} &
      \textbf{1.072} & \underline{0.488} & 0.476 & \textit{0.687} &
      \textbf{0.775} & \textbf{0.652} & \textbf{0.726} \\
    \bottomrule
  \end{tabular}
  }
  \endgroup
\end{table}

\subsection{Content-Only Review Is Not Reliably Achieved by Prompting Alone}
\label{sec:content_only_prompting}

We evaluate \textsc{Persistent}, a full-manuscript review protocol that repeatedly emphasizes scientific content and provides explicit evidence-based scoring guidance. Relative to \textsc{Standard}, its effects vary across backbones: only GPT-5.5 improves on all five robustness metrics. For Claude Sonnet 5, \textsc{Persistent} reduces MAD and Drift SD but lowers ICC, SPR, and discriminability, while the remaining backbones also show mixed changes. These results show that the evaluated content-focused prompting protocol does not consistently improve rhetorical robustness across backbones, motivating our investigation of an explicit content-normalized branch.

\subsection{False Robustness: Within-Paper Stability without Discrimination}
\label{sec:limited_robustness}

Gemini-3.5-Flash-Lite under \textsc{Standard} achieves the lowest MAD among the existing configurations, at 0.205, but its ICC is only 0.199 and its discriminability is 0.511, close to chance. The same pattern appears among specialized and agentic reviewers: DeepReviewer and OpenJudge show relatively low drift, yet attain ICC values of only 0.161 and 0.239 and discriminability values of 0.537 and 0.526. By contrast, GPT-5.5 under \textsc{Persistent} has a higher MAD of 0.476 but the strongest ICC, SPR, and discriminability among the 30 configurations. We call low rewrite-induced drift accompanied by weak paper differentiation \emph{false robustness}: invariance alone can create the appearance of robustness without the discrimination that robustness is meant to preserve.

\subsection{Human Alignment and Rhetorical Robustness Are Distinct}
\label{sec:human_alignment_robustness}

Within GPT-5.5, \textsc{Strict} provides the strongest human-alignment profile among the existing configurations, with a Human MAE of 1.078 and a Spearman correlation of 0.529, whereas \textsc{Persistent} has weaker human alignment but higher ICC, SPR, and discriminability. Human alignment asks whether a reviewer reproduces human judgments on observed manuscripts; rhetorical robustness asks whether judgments remain stable across matched rhetorical variants while preserving differences among papers. The two evaluations therefore favor different protocols, so human agreement cannot substitute for matched robustness evaluation.

\section{\oursm{}: Dual-Branch Review for Rhetorical Robustness}
\label{sec:method}

The central idea of \oursm{} is to augment manuscript-based review with a scientific judgment that is less coupled to rhetorical presentation. Unlike \textsc{Persistent}, which asks a reviewer to disregard rhetoric while still reading the complete manuscript, the science-core branch first transforms the input into a content-normalized scientific record. \oursm{} uses two complementary branches. The \emph{manuscript branch} reviews the complete paper under the \textsc{Strict} protocol, retaining the full manuscript context. The \emph{science-core branch} extracts a structured \emph{science core} containing the reported scientific record and reviews that record with an adapted protocol. The final overall assessment is the arithmetic mean of the two branch scores. This design preserves a conventional judgment of the paper while giving equal weight to a content-normalized judgment intended to vary less with rhetorical framing, organization, and linguistic expression.

\subsection{Design Principle: A Stable Scientific Branch without Discarding the Manuscript}
\label{sec:scicore_theory}

The science-core branch is designed to provide a scientific judgment that is less sensitive to rhetorical realization while retaining distinctions among papers. Conceptually, it maps different presentations of the same reported science to similar structured records without collapsing scientifically distinct manuscripts. This invariant-representation view motivates extracting and reviewing a science core, but the implemented extractor is only an approximation: its stability and paper separation must be established empirically rather than assumed.

The science-core branch complements rather than replaces manuscript review. When the science-core branch varies less across matched rhetorical realizations than the manuscript branch, averaging their scores can reduce rhetorical sensitivity while retaining the full paper context. This intuition addresses direct within-paper stability only; the final reviewer must still be evaluated using the joint stability-discrimination metrics to ensure that lower score drift does not come from cross-paper collapse.

We review the extracted record directly instead of first reconstructing another manuscript from it. Reconstruction cannot create decision-relevant information absent from its inputs and may introduce another rhetorical realization, but this information-theoretic argument does not guarantee better performance for a fixed LLM or order our implemented pipelines. We therefore treat direct review as a design choice and compare it empirically with ReconstructReview. Appendix~\ref{app:scicore_theory} gives the formal invariant-representation, conditional-stability, fusion, and Blackwell arguments.

\subsection{Science-Core Extraction: Isolating the Reported Scientific Record}
\label{sec:content_extraction}

The science-core branch first uses an LLM to extract a structured science core from the complete manuscript. The extraction focuses on information directly relevant to scientific evaluation, including the central research idea and claims, problem formulation, mathematical formulations and derivations, methods and assumptions, experimental or theoretical evidence, reported results, author-stated contributions, reproducibility information, and stated limitations.

The LLM is instructed to extract this information in objective, third-person language while preserving its attribution to the original manuscript. This is particularly important for scientific claims and author-stated contributions, where rhetorical framing could otherwise carry into the science core. The extractor is instructed to attribute claims and contributions to the authors, preserve reported results and numerical evidence, and avoid subjective assessments, evaluative language, emotional framing, and unsupported interpretations. It does not determine whether a claim is convincing, whether a contribution is significant, or how the paper should be scored.

The science core is produced as a text-based scientific record rather than a shortened rewrite of the manuscript. Figures are not retained directly. Tables are instead explicitly transcribed so that their reported values, comparisons, and other scientific information remain available to the downstream reviewer. The prompt asks the extractor to retain equations, mathematical arguments, reported numbers, and other details whenever they can be reliably recovered. When information cannot be located or read reliably, the extractor records this uncertainty rather than inferring or reconstructing the missing content.

\subsection{Science-Core Branch: Evaluating the Extracted Record}
\label{sec:report_review}

After extraction, the science-core branch reviews only the extracted record. The original manuscript is not provided to this branch at the review stage, so its judgment is based on the retained scientific record rather than its original rhetorical presentation.

We build the review protocol on the \textsc{Standard} protocol used in \oursb{}, while adapting its instructions to the structure of the science core. The prompt explains the role of each extracted section and how information across sections should be combined when forming a scientific judgment. In particular, the reviewer is instructed to connect the stated research problem and claims with the corresponding methods, assumptions, mathematical reasoning, and empirical or theoretical evidence; to assess whether the reported results support the author-stated claims and contributions; and to use the reproducibility information and stated limitations when evaluating the strength and scope of the evidence. Information that is explicitly absent from the manuscript can be treated as missing scientific support when relevant, while extraction uncertainty is handled separately.

The review retains the ICLR-style evaluation criteria and scoring scheme used by \textsc{Standard}. The reviewer produces written feedback, including a summary of the work, strengths, weaknesses, and questions, and an overall assessment, together with supporting scientific-evaluation scores. The score anchors follow the corresponding ICLR-style scales, with the review decision grounded in the scientific evidence contained in the science core.

\subsection{Dual-Branch Fusion: Integrating Manuscript and Science-Core Judgments}
\label{sec:scicore_fusion}

In parallel with the science-core branch, the manuscript branch applies the \textsc{Strict} protocol from Section~\ref{sec:reviewers} directly to the complete manuscript PDF. This branch retains the ordinary manuscript-level review context, while its evidence-focused rubric provides the conventional judgment used in the fusion. Both branches produce an overall assessment on the same ICLR-style scale.

Let $M(x)$ denote the manuscript-branch overall-assessment score and let $B(x)$ denote the science-core-branch score for manuscript realization $x$. The final \oursm{} score is their unweighted arithmetic mean:
\begin{equation}
S_{\oursm}(x)=\frac{1}{2}M(x)+\frac{1}{2}B(x).
\label{eq:scicore_fusion}
\end{equation}
Equal weighting provides a direct symmetric combination without introducing a tuned fusion parameter. The manuscript branch and science-core branch remain independently interpretable, which allows the evaluation to distinguish the behavior of each component from that of the fused reviewer.

\section{\oursm{}: Results and Analysis}
\label{sec:scicore_evaluation}

\subsection{Main Results: Rhetorical Robustness on \oursb{}}
\label{sec:report_review_evaluation}

Table~\ref{tab:benchmark_main} compares \oursm{} with the general-purpose, specialized, and agentic AI reviewers on \oursb{}. \oursm{} leads the primary comparison on ICC (0.775), SPR (0.652), and discriminability (0.726), while attaining the lowest Human MAE (1.072) and second-highest human Spearman correlation (0.488). However, it does not achieve the lowest MAD or Drift SD, and its human Spearman correlation remains below GPT-5.5 under \textsc{Strict}. \oursm{} therefore achieves a leading joint stability-discrimination profile while maintaining competitive human alignment. Paired paper-family bootstrap analysis supports its improvements over Manuscript-Strict across all five robustness metrics, alongside improved human alignment relative to Core-Adapted (Appendix~\ref{app:paper_family_bootstrap}). These results reinforce the complementary roles of the two branches and demonstrate the effectiveness of a simple equal-weight fusion.

\subsection{Dissecting the Gains: Architecture and Review-Policy Ablations}
\label{sec:understanding_robustness_gains}

Table~\ref{tab:gpt55_review_comparison} isolates the design choices using a common GPT-5.5 backbone. Manuscript-Standard and Manuscript-Strict review the PDF directly. ReconstructReview reconstructs a paper from the science core, permitted figures, and bibliography before applying \textsc{Standard} review (Appendix~\ref{app:reconstruct_review}). Four core-only configurations share the same cached science core and differ only in review policy: Standard, Strict, Persistent, or the adapted protocol. The final \oursm{} row averages Core-Adapted with Manuscript-Strict.

\textbf{Directly reviewing the science core provides a stronger robustness branch than reconstruction in this pipeline.} Relative to Manuscript-Standard, ReconstructReview worsens MAD from 0.598 to 0.651 and Drift SD from 0.956 to 1.101, with only modest gains on the joint metrics. Core-Standard then improves all seven metrics over ReconstructReview. The result shows that reconstruction does not recover the robustness of direct science-core review here. Proposition~\ref{prop:blackwell_reconstruction} provides an information-theoretic rationale for direct review but does not predict the ordering of the implemented pipelines.

\textbf{Content-focused review remains nontrivial after extraction.} Core-Persistent worsens all five robustness metrics relative to Core-Standard, whereas Core-Adapted achieves the lowest MAD and Drift SD and the highest ICC and SPR among the core-only configurations. The review policy thus affects the balance between score stability and paper discrimination even when the extracted representation is held fixed.

\textbf{The two branches provide complementary judgments.} Core-Adapted is more stable than Manuscript-Strict but aligns less closely with human scores. Fusion reduces MAD from 0.766 to 0.476 and increases ICC from 0.632 to 0.775, improving all five robustness metrics over Manuscript-Strict. It also improves human alignment, ICC, and discriminability over Core-Adapted, which retains lower MAD and Drift SD and higher SPR. Paired bootstrap intervals support these directions (Appendix~\ref{app:paper_family_bootstrap}).

With Manuscript-Strict and equal weighting fixed, the science-core branch
also improves all five robustness metrics over adding Standard or
Persistent manuscript review, supported by paired bootstrap intervals.
These controls share its two-score averaging rule and half-point
resolution. Appendix~\ref{app:manuscript_ensemble_controls} discusses the
mechanical effects of averaging and a three-review control matching
\oursm{}'s nominal call count. That control has lower Drift SD and higher
ICC, while \oursm{} has better point estimates on the other five metrics.

\begin{table}[!t]
\centering
\caption{\textbf{Within-backbone analysis of \oursm{} and its components.} Every configuration uses GPT-5.5. MAD and Drift SD are direct within-paper metrics, whereas ICC, SPR, and discriminability jointly relate within-paper consistency to cross-paper variation or separation. The final \oursm{} row averages the Manuscript-Strict and Core-Adapted scores. The best point estimate in each column is shown in bold and the second-best is underlined.}
\label{tab:gpt55_review_comparison}
\begingroup
\footnotesize
\setlength{\tabcolsep}{3pt}
\renewcommand{\arraystretch}{0.95}
\begin{tabular*}{\textwidth}{@{\extracolsep{\fill}}l|cc|cc|ccc@{}}
\toprule
\multirow{2}{*}{\textbf{Configuration}} &
\multicolumn{2}{c}{\textbf{Human alignment}} &
\multicolumn{2}{c}{\textbf{Within-paper stability}} &
\multicolumn{3}{c}{\textbf{Joint stability-discrimination}} \\
\cmidrule(lr){2-3}\cmidrule(lr){4-5}\cmidrule(lr){6-8}
& H-MAE $\downarrow$ & Spearman $\uparrow$ &
MAD $\downarrow$ & Drift SD $\downarrow$ & ICC $\uparrow$ &
SPR $\uparrow$ & Discrim. $\uparrow$ \\
\midrule
\multicolumn{8}{@{}l}{\itshape Manuscript-Based Review} \\
\cmidrule(lr){1-8}
Manuscript-Standard & 1.294 & 0.448 & 0.598 & 0.956 & 0.615 & 0.555 & 0.649 \\
Manuscript-Strict & \underline{1.078} & \textbf{0.529} & 0.766 & 1.202 & 0.632 & 0.507 & 0.672 \\
ReconstructReview & 1.490 & 0.133 & 0.651 & 1.101 & 0.623 & 0.566 & 0.664 \\
\midrule
\multicolumn{8}{@{}l}{\itshape Science-Core Branch} \\
\cmidrule(lr){1-8}
Core-Standard & 1.394 & 0.216 & \underline{0.354} & 0.737 & 0.734 & \underline{0.677} & \underline{0.697} \\
Core-Strict & 1.350 & 0.264 & 0.403 & 0.870 & 0.608 & 0.584 & 0.654 \\
Core-Persistent & 1.346 & 0.289 & 0.458 & 0.821 & 0.708 & 0.586 & 0.681 \\
Core-Adapted & 1.461 & 0.285 & \textbf{0.292} & \textbf{0.650} & \underline{0.751} & \textbf{0.710} & 0.695 \\
\midrule
\multicolumn{8}{@{}l}{\itshape Dual-Branch Fusion} \\
\cmidrule(lr){1-8}
\textbf{\oursm{} (ours)} & \textbf{1.072} & \underline{0.488} & 0.476 & \underline{0.687} & \textbf{0.775} & 0.652 & \textbf{0.726} \\
\bottomrule
\end{tabular*}
\endgroup
\end{table}

\subsection{Representation Analysis: Science-Core Stability and Paper Separation}
\label{app:representation_analysis}

A useful science-core representation should remain similar across rhetorical variants of the same paper while remaining distinct across different papers. Table~\ref{tab:report_similarity} compares each variant core with its matched original and with nonmatching originals across three extractors and two rewrite producers with cosine similarity of text embeddings.

Matched cores are consistently much more similar than cores from different papers. For GPT-5.5, matched similarity is 0.978, while cross-paper similarity ranges from 0.616 to 0.619. This gap provides embedding-level evidence that the extracted science cores remain stable across rhetorical rewrites while retaining paper-specific differences. Condition-level comparisons are reported in Appendix~\ref{app:condition_producer}.

\begin{table}[!t]
\centering
\caption{\textbf{Science-core cosine similarity across extractor models and rewrite producers.} All matched comparisons pair the science core extracted from each rewritten manuscript with the core extracted from the original manuscript of the same paper; different-paper controls pair it with the 59 nonmatching original cores. Each matched cell contains 600 pairs and each control cell contains 35,400 comparisons per producer. P05 denotes the fifth percentile of pair-level similarities.}
\label{tab:report_similarity}
\begingroup
\scriptsize
\setlength{\tabcolsep}{3.2pt}
\renewcommand{\arraystretch}{0.96}
\begin{tabular}{@{}llcccccc@{}}
\toprule
\multirow{2}{*}{\textbf{Extractor Model}} &
\multirow{2}{*}{\textbf{Comparison}} &
\multicolumn{3}{c}{\textbf{GPT-5.5 producer}} &
\multicolumn{3}{c}{\textbf{Opus 4.8 producer}} \\
\cmidrule(lr){3-5}\cmidrule(lr){6-8}
& & Mean & Median & P05 & Mean & Median & P05 \\
\midrule
\multirow{2}{*}{GPT-5-mini}
& All matched & 0.975 & 0.977 & 0.964 & 0.973 & 0.976 & 0.962 \\
& Different-paper control & 0.668 & 0.668 & 0.582 & 0.670 & 0.671 & 0.585 \\
\midrule
\multirow{2}{*}{GPT-5.5}
& All matched & 0.978 & 0.979 & 0.966 & 0.978 & 0.979 & 0.965 \\
& Different-paper control & 0.616 & 0.617 & 0.516 & 0.619 & 0.620 & 0.518 \\
\midrule
\multirow{2}{*}{Gemini-3.5-Flash-Lite}
& All matched & 0.963 & 0.965 & 0.934 & 0.961 & 0.964 & 0.934 \\
& Different-paper control & 0.606 & 0.606 & 0.506 & 0.609 & 0.608 & 0.509 \\
\bottomrule
\end{tabular}
\endgroup
\end{table}

\subsection{Criterion-Level Behavior of the Science-Core Branch}
\label{app:criterion_level}

Table~\ref{tab:secondary_score_ablation} shows that branch behavior is criterion-dependent. Science-core review generally improves soundness stability, but the contribution is less consistent, and human alignment does not improve uniformly. Presentation is diagnostic because manuscript-based review evaluates the paper's presentation, whereas science-core review evaluates the clarity and completeness of the extracted scientific record.

Confidence exposes a concrete false-robustness failure mode. Core-Standard and Core-Strict assign a constant confidence score, and Core-Persistent produces almost no variation, so their near-zero drift reflects output collapse and coincides with zero or near-zero SPR. Core-Adapted restores paper-level variation and raises confidence SPR to 0.394. This result again shows why within-paper stability cannot be interpreted without discrimination.

\begin{table}[!t]
\centering
\caption{\textbf{Secondary-score behavior across GPT-5.5 branch configurations.} For each secondary score, we report direct within-paper
stability (MAD), joint stability-discrimination (SPR), and Spearman
correlation with the corresponding mean human score. Presentation is marked as
diagnostic because its meaning differs between manuscript-based and direct
science-core review. A dash denotes an undefined Spearman correlation due to
constant scores; SPR is defined as zero for fully constant outputs.}
\label{tab:secondary_score_ablation}
\begingroup
\scriptsize
\setlength{\tabcolsep}{2.5pt}
\renewcommand{\arraystretch}{0.96}
\begin{tabular*}{\textwidth}{@{\extracolsep{\fill}}l*{12}{c}@{}}
\toprule
\multirow{2}{*}{\textbf{Configuration}} &
\multicolumn{3}{c}{\textbf{Soundness}} &
\multicolumn{3}{c}{\textbf{Contribution}} &
\multicolumn{3}{c}{\textbf{Presentation}$^{\dagger}$} &
\multicolumn{3}{c}{\textbf{Confidence}} \\
\cmidrule(lr){2-4}\cmidrule(lr){5-7}\cmidrule(lr){8-10}\cmidrule(lr){11-13}
& MAD $\downarrow$ & SPR $\uparrow$ & $\rho_{\mathrm{H}}$ $\uparrow$
& MAD $\downarrow$ & SPR $\uparrow$ & $\rho_{\mathrm{H}}$ $\uparrow$
& MAD $\downarrow$ & SPR $\uparrow$ & $\rho_{\mathrm{H}}$ $\uparrow$
& MAD $\downarrow$ & SPR $\uparrow$ & $\rho_{\mathrm{H}}$ $\uparrow$ \\
\midrule
\multicolumn{13}{@{}l}{\itshape Manuscript-Based Review} \\
\cmidrule(lr){1-13}
Manuscript-Standard
& 0.196 & 0.547 & 0.099 & 0.218 & 0.587 & 0.350
& 0.084 & 0.371 & -0.109 & 0.232 & 0.375 & -0.013 \\
ReconstructReview
& 0.112 & 0.625 & 0.182 & 0.214 & 0.507 & 0.301
& 0.240 & 0.375 & -0.094 & 0.289 & 0.402 & 0.052 \\
\midrule
\multicolumn{13}{@{}l}{\itshape Science-Core Branch} \\
\cmidrule(lr){1-13}
Core-Standard
& 0.135 & 0.645 & 0.046 & 0.139 & 0.593 & 0.226
& 0.197 & 0.524 & -0.078 & 0.000 & 0.000 & -- \\
Core-Strict
& 0.117 & 0.626 & 0.178 & 0.164 & 0.586 & 0.428
& 0.220 & 0.537 & 0.084 & 0.000 & 0.000 & -- \\
Core-Persistent
& 0.178 & 0.581 & 0.170 & 0.250 & 0.505 & 0.326
& 0.193 & 0.503 & 0.192 & 0.003 & 0.000 & -- \\
Core-Adapted
& 0.141 & 0.633 & 0.095 & 0.175 & 0.537 & 0.320
& 0.171 & 0.543 & 0.041 & 0.369 & 0.394 & 0.026 \\
\bottomrule
\end{tabular*}
\endgroup
\end{table}

\subsection{Cross-Backbone Generalization of the Science-Core Branch}
\label{sec:cross_backbone}

Table~\ref{tab:scicore_cross_backbone_main} compares Core-Adapted with Manuscript-Standard under the same backbone. We use \textsc{Standard} as a common baseline rather than selecting the strongest direct-review protocol separately for each model. Core-Adapted improves ICC for all three backbones. GPT-5.5 improves on all five robustness metrics, GPT-5-mini improves MAD, Drift SD, and ICC but weakens SPR and discriminability, and GLM-5.2 improves all seven reported metrics. Human alignment does not improve consistently.

\begin{table}[!t]
\centering
\caption{\textbf{Cross-backbone evaluation of the science-core branch.} For each backbone, Manuscript-Standard applies the \textsc{Standard} protocol to the full paper, whereas Core-Adapted uses the same backbone for science-core extraction and review. Bold marks the better result within each backbone pair.}
\label{tab:scicore_cross_backbone_main}
\begingroup
\scriptsize
\setlength{\tabcolsep}{2.2pt}
\renewcommand{\arraystretch}{0.96}
\begin{tabular*}{\textwidth}{@{\extracolsep{\fill}}ll|cc|cc|ccc@{}}
\toprule
\multirow{2}{*}{\textbf{Backbone}} & \multirow{2}{*}{\textbf{Configuration}} &
\multicolumn{2}{c}{\textbf{Human alignment}} &
\multicolumn{2}{c}{\textbf{Within-paper stability}} &
\multicolumn{3}{c}{\textbf{Joint stability-discrimination}} \\
\cmidrule(lr){3-4}\cmidrule(lr){5-6}\cmidrule(lr){7-9}
& & H-MAE $\downarrow$ & Spearman $\uparrow$ &
MAD $\downarrow$ & Drift SD $\downarrow$ & ICC $\uparrow$ &
SPR $\uparrow$ & Discrim. $\uparrow$ \\
\midrule
\multirow{2}{*}{GPT-5.5}
& Manuscript-Standard & \textbf{1.294} & \textbf{0.448} & 0.598 & 0.956 & 0.615 & 0.555 & 0.649 \\
& Core-Adapted & 1.461 & 0.285 & \textbf{0.292} & \textbf{0.650} & \textbf{0.751} & \textbf{0.710} & \textbf{0.695} \\
\cmidrule(lr){1-9}
\multirow{2}{*}{GPT-5-mini}
& Manuscript-Standard & \textbf{1.711} & \textbf{0.417} & 0.543 & 0.945 & 0.401 & \textbf{0.390} & \textbf{0.586} \\
& Core-Adapted & 1.794 & 0.290 & \textbf{0.367} & \textbf{0.866} & \textbf{0.408} & 0.371 & 0.577 \\
\cmidrule(lr){1-9}
\multirow{2}{*}{GLM-5.2}
& Manuscript-Standard & 2.044 & $-0.034$ & 1.500 & 2.341 & 0.226 & 0.322 & 0.558 \\
& Core-Adapted & \textbf{1.912} & \textbf{0.413} & \textbf{1.467} & \textbf{2.169} & \textbf{0.322} & \textbf{0.386} & \textbf{0.591} \\
\bottomrule
\end{tabular*}
\endgroup
\end{table}

The manuscript-side Strict evaluations for these backbones are already included in Table~\ref{tab:benchmark_main}. All reviewer configurations are evaluated on the same matched corpus, which includes rewrites produced independently by GPT-5.5 and Claude Opus 4.8. Because each experiment changes the backbone for both extraction and core review, it measures branch-level transfer without isolating which stage limits performance and does not test the final dual-branch fusion across backbones. The evidence therefore supports partial, backbone-dependent transfer of the science-core branch rather than a universal gain.

\subsection{Science-Core Weight Sensitivity}
\label{sec:scicore_weight_sensitivity}

As a sensitivity analysis of \oursm{}, we alter the science-core weight $\alpha$ to test how the fusion balance affects performance. We examine the science-core branch paired with Manuscript-Standard, Manuscript-Strict, and Manuscript-Persistent reviews:

\begin{equation}
S_{\alpha,p}(x)=\alpha B(x)+(1-\alpha)M_p(x),
\label{eq:scicore_weight_sweep}
\end{equation}

where $B(x)$ is the science-core-branch score and $M_p(x)$ is the manuscript-branch score under protocol $p$. Our prespecified configuration uses Manuscript-Strict with $\alpha=0.5$. The sweep characterizes sensitivity around this fixed design. For descriptive visualization, each metric is direction-aligned and min-max normalized over the full $\alpha\in[0,1]$ sweep.

\begin{center}
  \centering
  \includegraphics[width=0.99\linewidth]{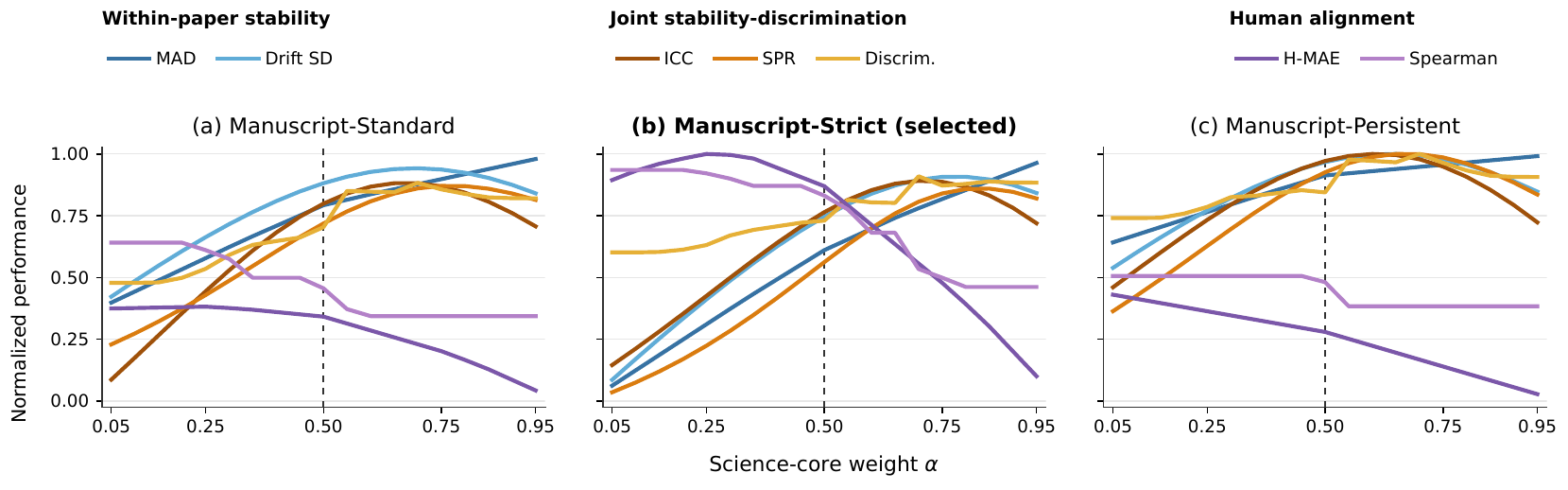}
  \captionof{figure}{\textbf{Science-core weight sensitivity.} Normalized performance across science-core weights under three manuscript protocols. Higher values indicate better performance for each metric. The dashed line marks the prespecified equal-weight fusion ($\alpha=0.5$). Complete numerical results are reported in Tables~\ref{tab:scicore_weight_standard}--\ref{tab:scicore_weight_persistent}.}
  \label{fig:scicore_fusion_weight_ablation}
\end{center}

Larger science-core weights tend to improve robustness while weakening human alignment overall. Across the three manuscript protocols, Manuscript-Persistent favors robustness most strongly, while Manuscript-Strict retains the best human alignment. Under Manuscript-Strict, $\alpha=0.5$ lies near the transition between these two objectives and maintains strong ICC, SPR, and discriminability without a large loss in human alignment.

Equal weighting is a simple and effective default that requires no search and naturally balances the science-core and manuscript branches. The prespecified equal-weight fusion balances rhetorical robustness and human alignment and is Pareto non-dominated among the evaluated protocol-weight combinations on the seven reported point estimates.

\section{Related Work}
\label{sec:related_work}

LLM-as-a-judge and AI peer-review research evaluates human agreement, review quality, score prediction, and workflow capability, while documenting sensitivity to order, length, prompts, and evaluator identity~\citep{liu2023g,zheng2023judging,wang2024large,dubois2024length,liang2024can,zhou2024llm}. Studies of automated review further reveal vulnerabilities to hidden instructions, artificial perturbations, and visible content-preserving revisions~\citep{ye2024we,lin2025breaking,kaneko2026paraphrasing,li2026gaming,baumann2026stop}. \citet{li2026rhetoric} characterize rhetorical reward hacking through controlled full-manuscript rewrites. Our work formulates rhetorical robustness as a joint requirement of within-paper stability and cross-paper discrimination, and introduces \oursm{} to improve this balance while maintaining competitive human alignment.

Existing interventions control particular biases or judge intermediate representations~\citep{dubois2024length,li2026rethinking}. Rather than relying only on instructions to discount rhetoric, \oursm{} augments full-manuscript review with a science-core branch that extracts and evaluates structured scientific content, then averages the two judgments. This retains manuscript-level assessment while adding a content-normalized view, motivated by invariant representation and Blackwell comparison~\citep{eaton1989group,dubois2021lossy,blackwell1953equivalent}. Appendix~\ref{app:related_work} provides the extended discussion.

\section{Conclusion}
\label{sec:conclusion}

We identify rhetorical robustness as an important but comparatively neglected requirement for trustworthy AI reviewers. Within-paper stability is necessary but not sufficient: low drift can be obtained by collapsing scores across papers, so robustness must also preserve between-paper discrimination. \oursb{} operationalizes this requirement with direct within-paper metrics and joint stability-discrimination metrics that relate same-paper consistency to cross-paper variation or separation. It shows that the evaluated content-focused prompting protocol produces model-dependent rather than consistent robustness gains, low score drift can conceal weak discrimination, and human alignment ranks reviewers differently from matched robustness evaluation. Content-focused review is therefore a nontrivial design problem that is not resolved by score-sensitivity or human-agreement measures alone.

\oursm{} provides one method for incorporating this requirement into reviewer design. It augments a conventional full-manuscript judgment with an equally weighted content-normalized judgment from a science-core branch. The extracted science cores are stable and paper-specific under our representation diagnostics, while the fused reviewer achieves a leading joint stability-discrimination profile and competitive human alignment in the primary GPT-5.5 comparison. Branch-level behavior remains criterion- and backbone-dependent. Together, the benchmark and method position rhetorical robustness as a distinct evaluation and design objective for more reliable AI reviewers.

\section{Limitations}

\oursb{} contains 60 ICLR 2026 submissions sampled evenly across score strata from papers with matched arXiv sources, so its results may not generalize to other venues, fields, or manuscript formats. The rhetorical rewrites are designed to preserve reported scientific content but cannot guarantee exact equivalence, particularly under complex conditions; the automated fidelity audit identifies a nonzero mismatch rate. Mean human scores provide only a limited external reference and do not capture disagreement among reviewers. The final \oursm{} fusion is evaluated with GPT-5.5, while the criterion-level and cross-backbone analyses characterize the science-core branch rather than the fused reviewer. Those branch experiments vary extraction and review together rather than isolating the two stages. Science-core extraction can omit or misread scientific details and requires additional inference, and equal-weight score fusion assumes that the two branch assessments are comparable on the shared review scale.

\bibliographystyle{assets/plainnat}
\bibliography{main}

\clearpage
\beginappendix
\section{Theoretical Motivation for \oursm{}}
\label{app:scicore_theory}

This section formalizes the representation and reconstruction arguments motivating the science-core branch. These results provide design rationales rather than performance guarantees for a fixed LLM reviewer. Throughout, $x\in\mathcal{X}$ denotes a generic manuscript realization. When we specialize the discussion to \oursb{}, $x_{i0}$ and $x_{ik}$ retain their definitions from Section~\ref{sec:problem_formulation}, and we suppress the reviewer-configuration index when a branch is fixed.

\subsection{The Maximal-Invariant Ideal}

We write $x\sim_{\mathrm{sci}}x'$ when two realizations express the same reported scientific record, including claims, methods, assumptions, evidence, results, and conclusions, while differing rhetorically. This relation defines the ideal invariance target. Our generated rewrites are designed to approximate it, and their preservation is assessed empirically rather than assumed to be exact.

An ideal representation $C^\star:\mathcal{X}\rightarrow\mathcal{Z}$ is a maximal invariant to rhetorical realization when
\begin{equation}
C^\star(x)=C^\star(x')
\quad\Longleftrightarrow\quad
x\sim_{\mathrm{sci}}x'.
\label{eq:scicore_maximal_invariant}
\end{equation}
The right-to-left implication expresses rhetorical invariance; the left-to-right implication prevents different scientific records from being collapsed into a single representation. We use maximal invariant with respect to the equivalence relation $\sim_{\mathrm{sci}}$. This equivalence-class view follows the standard maximal-invariant principle and its use for invariant downstream prediction~\citep{eaton1989group,dubois2021lossy}.

\begin{proposition}[Factorization through an ideal science core]
\label{prop:scicore_factorization}
If $C^\star$ satisfies Equation~\ref{eq:scicore_maximal_invariant}, every rhetorically invariant judgment $F:\mathcal{X}\rightarrow\mathcal{A}$ factors through $C^\star$: there exists a unique map $R:C^\star(\mathcal{X})\rightarrow\mathcal{A}$ such that $F=R\circ C^\star$. Conversely, every judgment of the form $R\circ C^\star$ is rhetorically invariant.
\end{proposition}

\begin{proof}
For $z=C^\star(x)$, define $R(z)=F(x)$. If $C^\star(x)=C^\star(x')$, maximality gives $x\sim_{\mathrm{sci}}x'$, and invariance of $F$ gives $F(x)=F(x')$; therefore $R$ is well defined. The converse follows from the invariance of $C^\star$.
\end{proof}

Proposition~\ref{prop:scicore_factorization} motivates the science-core branch: an invariant judgment can, in principle, operate on a representation of the scientific equivalence class rather than on a particular manuscript realization. Our implemented extractor $\widehat C$ is not claimed to be maximal or exactly invariant. It instead approximates the two properties in Equation~\ref{eq:scicore_maximal_invariant}: stability across matched variants designed to preserve the reported science and separation across different papers. The complete \oursm{} reviewer does not replace manuscript-based judgment with this representation. It uses the resulting content-normalized judgment as one branch of the final assessment.

\subsection{Conditional Stability of the Science-Core Branch and Fusion}

Specializing the downstream judgment to a numerical overall-assessment score, let $r:\mathcal{Z}\rightarrow\mathbb{R}$ denote the score map applied to the implemented science-core representation, and let $d_{\mathcal Z}$ be a task-relevant distance between extracted representations. If $r$ is $L$-Lipschitz on the observed representation range, then
\begin{equation}
\left|
r\!\left(\widehat C(x_{ik})\right)
-r\!\left(\widehat C(x_{i0})\right)
\right|
\leq
L\,d_{\mathcal Z}\!\left(
\widehat C(x_{ik}),\widehat C(x_{i0})
\right).
\label{eq:scicore_pairwise_stability}
\end{equation}
Averaging over the controlled comparisons gives the corresponding deterministic-score MAD bound
\begin{equation}
\mathrm{MAD}_{r\circ\widehat C}
\leq
\frac{L}{NK}
\sum_{i=1}^{N}\sum_{k=1}^{K}
d_{\mathcal Z}\!\left(
\widehat C(x_{ik}),\widehat C(x_{i0})
\right).
\label{eq:scicore_mad_bound}
\end{equation}
For stochastic reviewing, the same inequality applies to an $L$-Lipschitz conditional mean score map, while realized single-run scores include additional decoding variation. Equation~\ref{eq:scicore_mad_bound} should therefore be interpreted as a deterministic or conditional-mean analogue of the empirical MAD in Appendix~\ref{app:metric_definitions}, not as a bound on every realized review. The relation is not an empirical guarantee because we do not establish exact rewrite preservation, identify the reviewer's task-relevant metric, or estimate $L$. It addresses only direct within-paper stability; whether the reviewer retains distinctions among papers must still be assessed empirically through the joint stability-discrimination metrics.

Let $M(x)$ denote the manuscript-branch score and let $B(x)=r(\widehat C(x))$ denote the science-core-branch score. Applied to \oursb{}, $M(x_{ik})$ and $B(x_{ik})$ are the branch-specific instances of $y_{mik}$ in Section~\ref{sec:problem_formulation}, with the fixed configuration index suppressed. Under the fusion rule in Equation~\ref{eq:scicore_fusion}, the triangle inequality and Equation~\ref{eq:scicore_mad_bound} give
\begin{equation}
\mathrm{MAD}_{S_{\oursm}}
\leq
\frac{1}{2}\mathrm{MAD}_{M}
+\frac{1}{2}\mathrm{MAD}_{B}
\leq
\frac{1}{2}\mathrm{MAD}_{M}
+\frac{L}{2NK}
\sum_{i=1}^{N}\sum_{k=1}^{K}
d_{\mathcal Z}\!\left(\widehat C(x_{ik}),\widehat C(x_{i0})\right).
\label{eq:scicore_fusion_bound}
\end{equation}
This decomposition provides a route to attenuating rhetorical variation: when the science-core branch has lower rewrite-induced MAD than the manuscript branch, their average is bounded by a value below the manuscript branch's MAD. Stable extracted representations can contribute to this condition when the downstream score map is sufficiently regular, while the manuscript branch retains the complete paper context. The bound remains explanatory rather than a performance guarantee, and the fused reviewer must still be evaluated using the joint stability-discrimination metrics.

\subsection{Reconstruction as a Blackwell Garbling}

The invariant-representation view explains why a science core can support rhetorically stable judgment, but it does not by itself distinguish direct review from first realizing the record as another manuscript. Blackwell's comparison of statistical experiments supplies an information-theoretic rationale for avoiding an unnecessary reconstruction layer, rather than a performance guarantee for a fixed reviewer~\citep{blackwell1953equivalent}. Let $Z=\widehat C(X)$ be the extracted record, let $U$ collect any auxiliary assets supplied unchanged to reconstruction, and write $W=(Z,U)$ for the complete reconstruction input. If a reconstructed manuscript is sampled through a channel $\widetilde X\sim Q(\,\cdot\mid W)$, then for any review-relevant state $\Theta$,
\begin{equation}
\Theta\longrightarrow W\longrightarrow\widetilde X
\label{eq:blackwell_chain}
\end{equation}
forms a Markov chain.

\begin{proposition}[Reconstruction is a Blackwell garbling]
\label{prop:blackwell_reconstruction}
Observing the reconstruction input $W$ Blackwell-dominates observing the reconstructed manuscript $\widetilde X$. Consequently, across decision problems, the best attainable risk using $W$ is no worse than the best attainable risk using $\widetilde X$.
\end{proposition}

\begin{proof}
For any reconstruction-based decision rule $\delta_{\widetilde X}(a\mid\widetilde x)$, an observer of $W$ can first sample $\widetilde X$ from $Q$ and then apply that rule:
\[
\delta_W(a\mid w)
=
\int \delta_{\widetilde X}(a\mid\widetilde x)
Q(d\widetilde x\mid w).
\]
The two procedures induce the same conditional distribution of decisions given $\Theta$, so every risk attainable from $\widetilde X$ is attainable from $W$.
\end{proof}

Thus reconstruction cannot add decision-relevant information absent from its inputs, although it can discard information or introduce another rhetorical realization. This proposition establishes only the information ordering between the complete reconstruction input $W$ and its reconstructed output $\widetilde X$. It does not imply that a fixed LLM will review $W$ more effectively than $\widetilde X$. Moreover, the implemented reconstruction arm receives permitted figures and bibliography in addition to $Z$, whereas the science-core branch reviews only $Z$. The proposition therefore motivates the ReconstructReview ablation but does not order the two implemented pipelines; their comparison is empirical.

\clearpage
\section{Extended Related Work}
\label{app:related_work}

\subsection{Reliability and Robustness of LLM-as-a-Judge Systems}

LLM-as-a-judge methods operationalize evaluation through natural-language
rubrics, pairwise decisions, and task-specific evaluator models. Systems such
as G-Eval, MT-Bench, and Prometheus show that these designs can reproduce human
preferences and provide criterion-specific feedback
\citep{liu2023g,zheng2023judging,kim2024prometheus}. Their judgments are
nevertheless affected by candidate order, response length, prompt wording, and
evaluator identity \citep{wang2024large,dubois2024length,stureborg2024large,
panickssery2024llm}. Further variation arises from epistemic markers and
semantically equivalent evaluation instructions
\citep{lee2025llm,bhat2026all}, with meta-evaluation revealing substantial
dependence on the judge, task, property, and data source
\citep{bavaresco2025llms}.

Prior work has also begun to intervene on the evaluation process itself.
Length-controlled evaluation reduces a known presentation bias
\citep{dubois2024length}, while representation-based judging replaces direct
generation with an intermediate semantic representation
\citep{li2026rethinking}. These approaches suggest that robustness can depend
on what information reaches the judge, not only on which rubric the judge
receives. Existing studies, however, largely address generic evaluation tasks
or isolated biases. We instead test scientific reviewers for both stability
across rewrites designed to preserve reported scientific content and
discrimination across papers, and test whether a science-core branch improves
this balance when combined with manuscript-based review.

Our theoretical view draws on two complementary traditions. Maximal invariants
represent equivalence classes while retaining the information needed by
invariant downstream tasks \citep{eaton1989group,dubois2021lossy}. Blackwell's
comparison of experiments orders observations by the decision risks they make
attainable and identifies stochastic post-processing as a garbling of its
source \citep{blackwell1953equivalent}. We use these results as design
rationales for the science-core branch in \oursm{} and for avoiding manuscript
reconstruction within that branch, not as performance guarantees for a fixed LLM reviewer.

\subsection{LLMs for Scientific Peer Review}

Before modern LLMs, computational peer-review research focused on data
collection, decision prediction, and structured feedback generation. PeerRead linked manuscripts
with expert reports, aspect scores, and publication decisions for predictive
and analytical tasks \citep{kang2018dataset}. ReviewRobot and ReviewAdvisor
then moved toward evidence-linked and aspect-specific critique
\citep{wang2020reviewrobot,yuan2022can}. This progression expanded the modeled
review workflow beyond acceptance prediction, while leaving reliable,
decision-relevant criticism as a central challenge.

Modern LLMs extend this progression from modeling individual review components
to generating complete natural-language reviews at scale. In a large-scale
study, GPT-4 feedback overlaps with human review comments at rates comparable
to the overlap between human reviews, and many authors report finding such feedback useful
\citep{liang2024can}. More direct evaluations remain qualified: LLMs struggle
with long-paper processing, zero-shot scoring, and consistently correct
criticism \citep{zhou2024llm}; they reproduce summaries and stated strengths
more readily than substantive weaknesses, discriminating questions, or
differences in paper quality \citep{li2025unveiling}. Review quality also
depends on prompting, retrieval, and reviewer-guideline design
\citep{chen2026peercheck,li2026evaluating}. At the same time, a randomized
deployment shows that AI feedback can improve the specificity and actionability
of human reviews \citep{thakkar2026large}. Together, these results support AI as
a review aid but leave open whether its scientific judgments are stable under
content-preserving rhetorical changes.

Recent specialized models and agentic systems seek stronger reviewing through
fine-tuning, multi-stage reasoning, retrieval, verification, reviewer
ensembles, or proactive evidence gathering
\citep{idahl-ahmadi-2025-openreviewer,weng2025cycleresearcherimprovingautomatedresearch,zhu-etal-2025-deepreview,
lu2024aiscientistfullyautomated,openjudge2025,fang2026passivegenerationinvestigationproactive}. Their evaluations primarily
emphasize review quality, similarity to human feedback, score prediction, or
workflow capability. These goals are complementary to rhetorical robustness:
a review may be detailed and plausible yet still assign different scientific
scores to different presentations of the same work. Such sensitivity makes
merit judgments depend on rhetorical choices and undermines comparability
across submissions. We therefore evaluate
general-purpose LLM reviewers, specialized review models, and agentic review systems under
the same controlled matched-family design.

\subsection{Manipulation and Presentation Sensitivity in AI Scientific Review}

Scientific evaluation legitimately responds to clarity and exposition, but
these criteria should not be conflated with methodological soundness or
scientific contribution \citep{james2024rigour,li2025developing}. This
distinction is difficult to study with naturally occurring papers because
scientific merit and writing quality are entangled. Controlled interventions
provide a more direct test by varying presentation through revisions designed
to preserve the reported scientific content.

One line of work studies overtly adversarial manuscript interventions. Hidden
instructions can inflate ratings, suppress criticism, or redirect generated
reviews \citep{ye2024we,collu2026misleading}, while character-, word-, and
sentence-level perturbations can distort review judgments
\citep{lin2025breaking}. Such attacks establish serious vulnerabilities but
differ from ordinary academic rewriting because they rely on concealed
instructions or artificial perturbations. Related evidence also shows that
metadata and other manuscript-side factors can introduce bias into automated
review \citep{vasu2026justice,zhu2025your}.

A more closely related line examines visible revisions intended to preserve
reported scientific content.
Even title-level stylistic variants can alter AI-review scores
\citep{du-2025-titletrap}. At the abstract and full-paper levels, paraphrasing,
rewriting, and overclaiming can be optimized to improve automated ratings
\citep{kaneko2026paraphrasing,li2026gaming,wang2026emerging}, and recent work describes
full-manuscript \emph{paper laundering} or \emph{adversarial repackaging} that
raises review scores without new experiments
\citep{baumann2026stop,yang2026no}. More general attacks similarly learn
meaning-preserving stylistic edits that exploit a particular LLM judge
\citep{yang2026turning}.

\citet{li2026rhetoric} construct controlled full-manuscript rewrites to
characterize how rhetorical choices reward-hack AI reviewers. Using complete
manuscripts, our investigation evaluates within-paper stability, cross-paper
discrimination, and human alignment. We introduce \oursm{}, which combines
manuscript-level assessment with a content-normalized scientific judgment,
and examine direct science-core review against manuscript reconstruction. Complementarily,
\citet{dycke2026automatic}
changes substantive reasoning to test defect detection, whereas our
counterfactuals are constructed to preserve reported reasoning and evidence
when testing rhetorical invariance. Sensitivity to substantive defects and invariance to content-preserving rhetorical changes are complementary requirements for robust review.

\clearpage
\section{Experimental Details}
\label{app:experimental_details}

This section provides additional details on \oursb{} construction and
experiments, \oursm{}, and its ablation studies. Our rhetorical intervention
design, rewrite construction procedures, and manuscript-review protocols
follow the methodology of~\citet{li2026rhetoric}, with prompts instantiated
for the present benchmark. All manuscript rewrites and
model-review scores reported here were generated for this study.

\subsection{\oursb{} Composition and Sampling}
\label{app:benchmark_inventory}

We construct \oursb{} from ICLR 2026 submissions using metadata downloaded through the OpenReview API. We first identify submissions for which title matching yields a corresponding arXiv paper with an available source package. To ensure that the matched arXiv manuscript closely corresponds to the ICLR submission, we require the text similarity between the ICLR submission and at least one current or historical arXiv version to exceed 0.8. Among submissions satisfying these criteria, we stratify papers by their mean human overall-assessment rating. Using a fixed random seed of 42, we randomly sample 10 papers from each of the six rating intervals $[1,3)$, $[3,4)$, $[4,5)$, $[5,6)$, $[6,7)$, and $[7,8.5]$, yielding 60 source papers. Equal allocation across these intervals ensures coverage of lower-, middle-, and higher-rated submissions. We retain the mean human ratings as the reference for the human-alignment evaluation. The benchmark focuses on ICLR 2026 submissions with matched arXiv sources. Generalization to other venues, fields, and manuscript formats remains to be evaluated.

We construct 10 rhetorical rewrite conditions for each source paper. Six conditions each apply a positive-direction transformation to a single rhetorical dimension. \emph{Novelty stance} changes how strongly the novelty of the work is presented. \emph{Scope framing} changes how broadly or narrowly the scope and generalizability of the work are presented. \emph{Evidence framing} changes how the reported evidence and quantitative results are characterized. \emph{Contribution salience} changes the prominence and organization of the stated contributions. \emph{Technical register} changes the style and degree of technical and formal presentation. \emph{Linguistic complexity} changes lexical and syntactic realization. The remaining 4 conditions introduce broader transformations. \emph{Joint rewrite} varies multiple rhetorical dimensions together. \emph{Recursive rewrite (R2)} and \emph{recursive rewrite (R3)} correspond to 2 and 3 successive rounds of rewriting, respectively. \emph{Reviewer-guided rewrite} uses reviewer feedback to guide the rhetorical transformation.

Each condition contains 2 independently generated variants for each of the 60 source papers, yielding 120 variants per condition and 1,200 rhetorical variants across all 10 conditions. Together with the 60 original manuscripts, \oursb{} contains 1,260 full-manuscript PDFs. During evaluation, each variant is paired with the corresponding original manuscript from the same source-paper family, yielding 120 matched original-variant pairs per condition and 1,200 matched pairs overall.

\subsection{Rewrite Construction and Structural Controls}
\label{app:rewrite_construction}

Each rhetorical condition is instantiated independently by GPT-5.5 through Codex CLI and Claude Opus 4.8~\citep{anthropic_claude_opus48} through Claude Code. Rewriting is performed on the complete \LaTeX{} project associated with each source paper rather than on isolated sections or extracted text. The transformation instructions require the models to preserve the reported scientific content, including claims, methods, evidence, numerical results, findings, and conclusions, while allowing the targeted changes in rhetorical framing, organization, wording, and presentation, including how quantitative evidence or tables are described and presented.

Citations, cross-references, figures, bibliography files, and other project dependencies are protected during rewriting. Each transformed project is subsequently recompiled to verify structural validity. Outputs that fail structural or compilation checks are repaired when possible and otherwise excluded from the benchmark.
These controls are designed to permit substantial changes in presentation
while minimizing unintended changes to the reported scientific content.

\subsection{Reviewer Execution Details}
\label{app:reviewer_execution}

We evaluate eight general-purpose LLMs prompted as scientific reviewers:
GPT-5.5~\citep{openai_gpt55}, GPT-5-mini~\citep{openai_gpt5mini}, Claude
Sonnet 5~\citep{anthropic_claude_sonnet5}, GLM-5.2~\citep{glm5team2026glm5vibecodingagentic},
Kimi-K2.6~\citep{kimi_k2_6}, GPT-OSS-120B~\citep{openai2025gptoss120bgptoss20bmodel},
Gemini-3.5-Flash-Lite~\citep{google_gemini35flashlite}, and Qwen-3.5-Flash~\citep{qwen3.5flash}.
This reflects a common use case in which a general-purpose LLM is directly
instructed to review a scientific manuscript. Each model is evaluated under
three protocols. \textsc{Standard} follows the standard ICLR review criteria
and scoring scheme. \textsc{Strict} and \textsc{Persistent} are derived from
\textsc{Standard}: \textsc{Strict} applies a more demanding, evidence-focused
rubric, and \textsc{Persistent} repeatedly
emphasizes that scientific judgments should depend on substantive scientific
content rather than rhetorical presentation. We additionally evaluate three
specialized scientific-review models: OpenReviewer~\citep{idahl-ahmadi-2025-openreviewer},
CycleReviewer~\citep{weng2025cycleresearcherimprovingautomatedresearch}, and
DeepReviewer~\citep{zhu-etal-2025-deepreview}. We also evaluate three agentic
review systems: AI Scientist~\citep{lu2024aiscientistfullyautomated},
OpenJudge~\citep{openjudge2025}, and ProReviewer~\citep{fang2026passivegenerationinvestigationproactive}.
The specialized and agentic systems use their native review procedures.

All prompted-review experiments operate directly on the manuscript PDFs. GPT-5.5 and GPT-5-mini are accessed through the OpenAI Responses API,\footnote{\url{https://platform.openai.com/docs/api-reference/responses}} using the default API settings without overriding sampling or generation parameters. Claude Sonnet 5, GLM-5.2, Kimi-K2.6, GPT-OSS-120B, Gemini-3.5-Flash-Lite, and Qwen-3.5-Flash are accessed through the OpenRouter API.\footnote{\url{https://openrouter.ai/docs/features/multimodal/pdfs}} For GLM-5.2, Kimi-K2.6, and GPT-OSS-120B, the reasoning effort is set to \texttt{high}. All other generation and inference parameters are left at their OpenRouter defaults. Each review request consists of the complete manuscript PDF together with the corresponding review prompt. For models accessed through OpenRouter, PDF ingestion and processing are handled by OpenRouter's file-processing pipeline. 

The specialized reviewer systems are evaluated using their officially released checkpoints and inference procedures. Specifically, we use Llama-OpenReviewer-8B for OpenReviewer, CycleReviewer-ML-Llama-3.1-8B for CycleReviewer, and DeepReviewer-14B for DeepReviewer. CycleReviewer and DeepReviewer follow their released multi-reviewer evaluation procedures, with the resulting reviewer scores aggregated according to their native implementations. These models are served locally on a server equipped with 8 NVIDIA A100 GPUs.

For specialized reviewers whose released interfaces do not accept PDF input,
we first convert each manuscript to Markdown with the Datalab OCR API\footnote{
\url{https://documentation.datalab.to/}} and pass the resulting Markdown to
the model's native inference pipeline.

The agentic review systems, AI Scientist, OpenJudge, and ProReviewer, are evaluated using their officially released implementations and configurations. We retain the model backbones, inference parameters, review workflows, and score extraction procedures specified by their respective implementations without additional modification.

\textbf{Single-review execution.} We use one independent review per manuscript version in the main experiments, as repeated reviewing would substantially increase the computational cost of \oursb{}. We conduct an auxiliary audit with GPT-5-mini under the Standard protocol, covering all 60 source papers, both rewrite producers, and the six single-dimension rewrite conditions, for a total of 720 rewritten manuscripts. For each manuscript, we compare a single review with the mean of three independent reviews of the same PDF. Averaging three reviews changes the mean OA from 6.493 to 6.457. The 12 condition--producer mean effect estimates are closely aligned between the two settings, with Pearson and Spearman correlations of 0.975 and 0.944, respectively. The mean absolute change in these effects is 0.056 OA points, with a maximum change of 0.117. This audit supports the consistency of condition-level mean effects under the evaluated configuration. Our robustness metrics characterize observed score variation under single-review execution, including stochastic generation variability.

Appendix~\ref{app:repeated_review_noise} reports the same-PDF repeatability baseline for the manuscript protocols and \oursm{}.

\subsection{Full Evaluation Metric Definitions}
\label{app:metric_definitions}

All configurations are evaluated on the same 60 complete paper families (1,260 manuscript versions), with no missing overall-assessment scores.

For reviewer configuration $m$ and a given score dimension, we define the
rewrite-induced score change as
\begin{equation}
\Delta_{mik}=y_{mik}-y_{mi0},
\qquad k=1,\ldots,K.
\label{eq:score_drift}
\end{equation}

\textbf{Direct within-paper stability.}
MAD measures the average magnitude of score changes induced by rhetorical rewriting:
\begin{equation}
\mathrm{MAD}_m
=
\frac{1}{NK}
\sum_{i=1}^{N}
\sum_{k=1}^{K}
\left|\Delta_{mik}\right|.
\label{eq:mad}
\end{equation}
Drift SD measures the variability of the signed score changes:
\begin{equation}
\mathrm{DriftSD}_m
=
\operatorname{SD}_{i,k}
\left(\Delta_{mik}\right).
\label{eq:drift_sd}
\end{equation}
Lower values indicate greater rhetorical stability. MAD captures the overall
magnitude of score drift, whereas Drift SD captures its heterogeneity across
papers and rewrite conditions.

\textbf{Joint stability-discrimination.}
Direct within-paper stability is insufficient when low drift is produced by
score collapse. We therefore use three complementary metrics that relate
within-paper consistency to cross-paper variation or separation. These are
joint metrics rather than between-paper-only measures. ICC measures
whether different presentations of the same paper receive similar scores
relative to differences across papers. We use the two-way mixed-effects,
absolute-agreement, single-measure form, ICC(A,1):
\begin{equation}
\mathrm{ICC}_m(A,1)
=
\frac{MS_R-MS_E}
{MS_R+(P-1)MS_E+\frac{P}{N}(MS_C-MS_E)},
\label{eq:icc}
\end{equation}
where $MS_R$, $MS_C$, and $MS_E$ denote the paper, presentation, and residual
mean squares, respectively; $N$ is the number of papers; and $P=K+1$ is the
number of presentations including the original manuscript. Higher ICC
indicates that between-paper differences are large relative to presentation-
induced and residual variation.

We further define the signal preservation ratio (SPR) as
\begin{equation}
\mathrm{SPR}_m
=
\frac{
\operatorname{Var}_{i}\left(y_{mi0}\right)
}{
\operatorname{Var}_{i}\left(y_{mi0}\right)
+
\mathbb{E}_{i,k}
\left[\Delta_{mik}^{2}\right]
}.
\label{eq:spr}
\end{equation}
SPR compares the cross-paper score variation in the original manuscripts with
the magnitude of rewrite-induced movement. Higher values indicate that
cross-paper variation remains large relative to within-paper rhetorical
perturbation.
When both the original-paper score variance and the rewrite-induced mean
squared drift are zero, we define SPR as zero, reflecting the absence of
between-paper discrimination.

Finally, discriminability measures whether presentations of the same paper are
closer in score than presentations of different papers:
\begin{equation}
\begin{aligned}
\mathrm{Disc}_m
=
\mathbb{E}_{\substack{
i\neq j,\; a\neq b \\
a,b,c\in\{0,\ldots,K\}
}}
\Big[
&\mathbf{1}
\left(
|y_{mia}-y_{mib}|
<
|y_{mia}-y_{mjc}|
\right)
\\
&+
\frac{1}{2}
\mathbf{1}
\left(
|y_{mia}-y_{mib}|
=
|y_{mia}-y_{mjc}|
\right)
\Big].
\end{aligned}
\label{eq:discriminability}
\end{equation}
Here, $a$ and $b$ index two presentations of paper $i$, while $c$ indexes a
presentation of a different paper $j$. A value of $0.5$ is chance-like, while
higher values indicate stronger separation between same-paper and different-
paper presentations.

ICC, SPR, and discriminability operationalize the same joint requirement in
different ways: each rewards consistency across presentations of the same
paper only when distinctions across papers remain detectable. None treats
cross-paper score variation as ground-truth scientific merit.

\textbf{Human judgment alignment.}
Let $h_i$ denote the mean human OA score for the original version of paper
$i$. Human MAE measures absolute agreement between AI and human scores:
\begin{equation}
\mathrm{HMAE}_m
=
\frac{1}{N}
\sum_{i=1}^{N}
\left|y_{mi0}-h_i\right|.
\label{eq:human_mae}
\end{equation}
We additionally measure rank agreement using Spearman rank correlation:
\begin{equation}
\rho^{\mathrm{human}}_m
=
\rho_{\mathrm{S}}
\left(
\left(y_{mi0}\right)_{i=1}^{N},
\left(h_i\right)_{i=1}^{N}
\right).
\label{eq:human_spearman}
\end{equation}
Lower Human MAE and higher Spearman rank correlation indicate stronger
agreement with human reviewer judgments. These metrics assess agreement with mean human ratings and do not characterize disagreement among individual reviewers.
Spearman correlation is undefined when either input is constant; ICC is
undefined when its denominator is zero. We report these cases as dashes.

\subsection{\oursm Execution Protocol}
\label{app:scicore_execution}

The Core-Adapted prompt, the Manuscript-Strict branch, and the equal-weight fusion rule were fixed before examining results on the 60-paper evaluation panel.

The \oursm{} experiments follow the same model-access and inference settings as the corresponding direct-review experiments described above. The science-core branch first uses GPT-5.5 to extract the science core from the manuscript and then reviews the extracted record with the adapted protocol. The manuscript branch is the GPT-5.5 \textsc{Strict} review of the complete PDF reported in the main benchmark. The final overall assessment is computed as the unweighted arithmetic mean of the manuscript-branch and science-core-branch scores. Both branches use the same ICLR overall-assessment scale, and their arithmetic fusion assumes that scores are comparable across branches. No additional model call is required for fusion. In the branch-policy ablation study, each extracted science core is reused across the evaluated policies; the reviewer backbone and input representation remain the same while the review instructions vary.

\subsection{ReconstructReview Ablation}
\label{app:reconstruct_review}

ReconstructReview evaluates whether the extracted record should be reviewed directly within the science-core branch or first realized as another manuscript. For each manuscript presentation, GPT-5.5 first extracts the science core using the same extraction procedure as \oursm{}. The extracted content is then provided to GPT-5.5 through Codex, which reconstructs it into a complete anonymous ICLR 2026 submission using the official conference template. The reconstruction workspace contains the extracted science core as the authoritative scientific record, together with the permitted bibliography, available scientific figures, and the ICLR 2026 template. The original manuscript prose is not provided to the reconstruction model.

The reconstruction model is given substantial authorial freedom over the title, framing, organization, mathematical presentation, allocation between the main text and appendix, and the selection and placement of figures and tables. At the same time, it is explicitly instructed to preserve the scientific record, retain nonredundant methods and experimental evidence, preserve reported numerical results and qualifications, and avoid introducing unsupported experiments, claims, citations, or implementation details. The resulting \LaTeX{} project is compiled into a full-manuscript PDF and reviewed using the Standard protocol.

We apply this reconstruction procedure independently to every presentation in \oursb{}, including the original baseline and the variants produced by both rewrite models under all 10 rhetorical conditions. This yields 21 reconstructed manuscripts per source paper and 1,260 reconstructed manuscripts in total. The resulting scores are then analyzed using the same matched original-variant evaluation procedure as in \oursb{}.

\clearpage
\section{Complete \oursb{} Results}
\label{app:complete_benchmark_results}

This section reports the complete baseline evidence supporting the compact
comparisons in the main paper. Baseline reviewer models and protocols are kept
separate from the backbone-transfer experiments for \oursm.

\subsection{Full Score-Dimension Results}
\label{app:score_dimension_results}

Complete OA results are already reported in Table~\ref{tab:benchmark_main} and
are not repeated here. The tables below use the same seven metrics and column
order as the main table, but restrict the evaluated target to soundness,
presentation, contribution, or confidence. Human alignment is computed against
the mean score from the official reviews of each \oursb{} source paper. Presentation
scores were recovered directly from the OpenReview API; the other three
secondary-score means exactly match the archived metadata. MAD and Drift SD
directly measure within-paper stability, whereas ICC, SPR, and discriminability
are joint stability-discrimination metrics. The GPT-5-mini entries are
recomputed from their complete paper-by-rewrite score matrices rather than from
the earlier aggregate-only export. External reviewers and the final
\oursm{} fusion expose only OA-compatible outputs and are therefore not included
in these secondary-score tables.

\clearpage
\begin{table}[!t]
  \centering
  \caption{\textbf{Full Soundness results for secondary scores.} Metrics and column order match \Cref{tab:benchmark_main}. The best result in each column is shown in bold, the second-best is underlined, and the third-best is italicized.}
  \label{tab:benchmark_soundness_full}
  \begingroup
  \resizebox{\linewidth}{!}{%
  \begin{tabular}{@{}ll|cc|cc|ccc@{}}
    \toprule
    \multirow{2}{*}{\textbf{System}} & \multirow{2}{*}{\textbf{Protocol}} &
    \multicolumn{2}{c}{\textbf{Human alignment}} &
    \multicolumn{2}{c}{\textbf{Within-paper stability}} &
    \multicolumn{3}{c}{\textbf{Joint stability-discrimination}} \\
    \cmidrule(lr){3-4}\cmidrule(lr){5-6}\cmidrule(lr){7-9}
    & & H-MAE $\downarrow$ & Spearman $\uparrow$ &
    MAD $\downarrow$ & Drift SD $\downarrow$ & ICC $\uparrow$ &
    SPR $\uparrow$ & Discrim. $\uparrow$ \\
    \midrule
    \multicolumn{9}{@{}l}{\itshape General-purpose LLM reviewers} \\
    \cmidrule(lr){1-9}
      \multirow[t]{3}{*}{GPT-5.5} & Standard & 0.529 & 0.099 & 0.196 & 0.432 & 0.520 & \textit{0.547} & 0.614 \\
       & Strict & 0.583 & 0.273 & 0.272 & 0.494 & 0.494 & 0.441 & 0.627 \\
       & Persistent & 0.549 & 0.070 & 0.227 & 0.481 & 0.531 & 0.532 & 0.619 \\
    \cmidrule(lr){1-9}
      \multirow[t]{3}{*}{GPT-5-mini} & Standard & \textbf{0.401} & 0.214 & \textbf{0.071} & \textbf{0.262} & 0.061 & 0.401 & 0.501 \\
       & Strict & 0.539 & 0.222 & 0.334 & 0.536 & 0.384 & 0.428 & 0.584 \\
       & Persistent & \underline{0.459} & 0.183 & 0.179 & 0.411 & 0.209 & 0.486 & 0.512 \\
    \cmidrule(lr){1-9}
      \multirow[t]{3}{*}{Claude Sonnet 5} & Standard & \textit{0.464} & \textbf{0.368} & 0.247 & 0.491 & 0.541 & 0.528 & \textit{0.631} \\
       & Strict & 0.570 & 0.175 & 0.208 & 0.448 & \underline{0.577} & \textbf{0.576} & \textbf{0.644} \\
       & Persistent & 0.578 & 0.196 & 0.208 & 0.450 & \textit{0.566} & 0.535 & \underline{0.642} \\
    \cmidrule(lr){1-9}
      \multirow[t]{3}{*}{GLM-5.2} & Standard & 0.534 & 0.118 & 0.477 & 0.867 & 0.222 & 0.314 & 0.540 \\
       & Strict & 0.778 & $-0.116$ & 0.624 & 0.894 & 0.211 & 0.385 & 0.549 \\
       & Persistent & 0.836 & $-0.222$ & 0.677 & 1.013 & 0.227 & 0.385 & 0.547 \\
    \cmidrule(lr){1-9}
      \multirow[t]{3}{*}{Kimi-K2.6} & Standard & 0.550 & 0.194 & 0.430 & 0.786 & 0.175 & 0.404 & 0.530 \\
       & Strict & 0.766 & 0.008 & 0.487 & 0.735 & 0.229 & 0.380 & 0.552 \\
       & Persistent & 0.667 & 0.007 & 0.469 & 0.714 & 0.192 & 0.391 & 0.543 \\
    \cmidrule(lr){1-9}
      \multirow[t]{3}{*}{GPT-OSS-120B} & Standard & 0.500 & $-0.286$ & \textit{0.167} & \textit{0.405} & 0.051 & 0.376 & 0.503 \\
       & Strict & 0.649 & $-0.063$ & 0.424 & 0.654 & 0.136 & 0.348 & 0.532 \\
       & Persistent & 0.498 & 0.107 & 0.370 & 0.621 & 0.165 & 0.391 & 0.531 \\
    \cmidrule(lr){1-9}
      \multirow[t]{3}{*}{Gemini-3.5-Flash-Lite} & Standard & 1.076 & 0.212 & \underline{0.167} & \underline{0.391} & 0.256 & 0.490 & 0.522 \\
       & Strict & 0.612 & \underline{0.318} & 0.338 & 0.511 & 0.400 & 0.437 & 0.599 \\
       & Persistent & 0.893 & \textit{0.312} & 0.296 & 0.508 & 0.365 & 0.477 & 0.563 \\
    \cmidrule(lr){1-9}
      \multirow[t]{3}{*}{Qwen-3.5-Flash} & Standard & 0.943 & $-0.105$ & 0.323 & 0.562 & 0.260 & 0.418 & 0.555 \\
       & Strict & 0.621 & 0.220 & 0.359 & 0.607 & 0.226 & 0.360 & 0.548 \\
       & Persistent & 0.803 & 0.212 & 0.338 & 0.560 & \textbf{0.632} & \underline{0.566} & 0.590 \\
    \bottomrule
  \end{tabular}
  }
  \endgroup
\end{table}

\begin{table}[!t]
  \centering
  \caption{\textbf{Full Presentation results for secondary scores.} Metrics and column order match \Cref{tab:benchmark_main}. The best result in each column is shown in bold, the second-best is underlined, and the third-best is italicized. A dash denotes an undefined Spearman correlation.}
  \label{tab:benchmark_presentation_full}
  \begingroup
  \resizebox{\linewidth}{!}{%
  \begin{tabular}{@{}ll|cc|cc|ccc@{}}
    \toprule
    \multirow{2}{*}{\textbf{System}} & \multirow{2}{*}{\textbf{Protocol}} &
    \multicolumn{2}{c}{\textbf{Human alignment}} &
    \multicolumn{2}{c}{\textbf{Within-paper stability}} &
    \multicolumn{3}{c}{\textbf{Joint stability-discrimination}} \\
    \cmidrule(lr){3-4}\cmidrule(lr){5-6}\cmidrule(lr){7-9}
    & & H-MAE $\downarrow$ & Spearman $\uparrow$ &
    MAD $\downarrow$ & Drift SD $\downarrow$ & ICC $\uparrow$ &
    SPR $\uparrow$ & Discrim. $\uparrow$ \\
    \midrule
    \multicolumn{9}{@{}l}{\itshape General-purpose LLM reviewers} \\
    \cmidrule(lr){1-9}
      \multirow[t]{3}{*}{GPT-5.5} & Standard & 0.441 & $-0.109$ & 0.084 & 0.287 & 0.232 & 0.371 & 0.510 \\
       & Strict & \textbf{0.393} & 0.220 & \underline{0.044} & \textit{0.208} & \textit{0.306} & 0.271 & 0.515 \\
       & Persistent & \textit{0.402} & \textit{0.285} & 0.085 & 0.284 & \underline{0.402} & 0.440 & 0.529 \\
    \cmidrule(lr){1-9}
      \multirow[t]{3}{*}{GPT-5-mini} & Standard & 0.410 & 0.181 & 0.114 & 0.324 & 0.263 & 0.220 & 0.526 \\
       & Strict & \underline{0.402} & -- & \textit{0.044} & \underline{0.205} & 0.067 & 0.000 & 0.502 \\
       & Persistent & 0.446 & 0.043 & 0.098 & 0.308 & 0.155 & 0.453 & 0.506 \\
    \cmidrule(lr){1-9}
      \multirow[t]{3}{*}{Claude Sonnet 5} & Standard & 0.491 & 0.033 & 0.185 & 0.415 & 0.241 & 0.447 & 0.526 \\
       & Strict & 0.410 & 0.142 & 0.072 & 0.264 & 0.186 & 0.398 & 0.508 \\
       & Persistent & 0.410 & 0.182 & 0.110 & 0.323 & 0.091 & 0.410 & 0.504 \\
    \cmidrule(lr){1-9}
      \multirow[t]{3}{*}{GLM-5.2} & Standard & 0.574 & $-0.044$ & 0.464 & 0.864 & 0.201 & 0.309 & 0.535 \\
       & Strict & 0.824 & $-0.243$ & 0.628 & 0.985 & 0.165 & 0.407 & 0.532 \\
       & Persistent & 0.807 & $-0.195$ & 0.673 & 1.033 & 0.222 & 0.371 & \underline{0.544} \\
    \cmidrule(lr){1-9}
      \multirow[t]{3}{*}{Kimi-K2.6} & Standard & 0.637 & 0.040 & 0.517 & 0.851 & 0.190 & 0.406 & 0.531 \\
       & Strict & 0.707 & $-0.159$ & 0.582 & 0.864 & 0.220 & 0.354 & \textbf{0.548} \\
       & Persistent & 0.804 & $-0.099$ & 0.588 & 0.834 & 0.213 & 0.390 & 0.542 \\
    \cmidrule(lr){1-9}
      \multirow[t]{3}{*}{GPT-OSS-120B} & Standard & 0.468 & $-0.169$ & 0.155 & 0.396 & 0.037 & 0.294 & 0.502 \\
       & Strict & 0.529 & $-0.086$ & 0.333 & 0.601 & 0.062 & 0.316 & 0.507 \\
       & Persistent & 0.510 & $-0.037$ & 0.302 & 0.550 & 0.090 & 0.379 & 0.516 \\
    \cmidrule(lr){1-9}
      \multirow[t]{3}{*}{Gemini-3.5-Flash-Lite} & Standard & 1.196 & 0.216 & \textbf{0.016} & \textbf{0.125} & 0.135 & \textbf{0.509} & 0.501 \\
       & Strict & 1.046 & \textbf{0.364} & 0.156 & 0.374 & 0.137 & \textit{0.471} & 0.506 \\
       & Persistent & 1.163 & 0.267 & 0.047 & 0.212 & 0.105 & \underline{0.504} & 0.501 \\
    \cmidrule(lr){1-9}
      \multirow[t]{3}{*}{Qwen-3.5-Flash} & Standard & 0.963 & \underline{0.304} & 0.367 & 0.607 & 0.105 & 0.382 & 0.518 \\
       & Strict & 0.716 & $-0.014$ & 0.498 & 0.688 & 0.074 & 0.318 & 0.519 \\
       & Persistent & 0.740 & 0.148 & 0.499 & 0.665 & \textbf{0.519} & 0.454 & \textit{0.544} \\
    \bottomrule
  \end{tabular}
  }
  \endgroup
\end{table}

\begin{table}[!t]
  \centering
  \caption{\textbf{Full Contribution results for secondary scores.} Metrics and column order match \Cref{tab:benchmark_main}. The best result in each column is shown in bold, the second-best is underlined, and the third-best is italicized.}
  \label{tab:benchmark_contribution_full}
  \begingroup
  \resizebox{\linewidth}{!}{%
  \begin{tabular}{@{}ll|cc|cc|ccc@{}}
    \toprule
    \multirow{2}{*}{\textbf{System}} & \multirow{2}{*}{\textbf{Protocol}} &
    \multicolumn{2}{c}{\textbf{Human alignment}} &
    \multicolumn{2}{c}{\textbf{Within-paper stability}} &
    \multicolumn{3}{c}{\textbf{Joint stability-discrimination}} \\
    \cmidrule(lr){3-4}\cmidrule(lr){5-6}\cmidrule(lr){7-9}
    & & H-MAE $\downarrow$ & Spearman $\uparrow$ &
    MAD $\downarrow$ & Drift SD $\downarrow$ & ICC $\uparrow$ &
    SPR $\uparrow$ & Discrim. $\uparrow$ \\
    \midrule
    \multicolumn{9}{@{}l}{\itshape General-purpose LLM reviewers} \\
    \cmidrule(lr){1-9}
      \multirow[t]{3}{*}{GPT-5.5} & Standard & 0.513 & 0.350 & 0.218 & 0.442 & \textit{0.600} & \textbf{0.587} & 0.620 \\
       & Strict & \underline{0.407} & 0.409 & 0.181 & 0.416 & 0.583 & 0.520 & \textit{0.626} \\
       & Persistent & \textit{0.416} & \textit{0.493} & 0.217 & 0.455 & \underline{0.629} & \underline{0.583} & \textbf{0.651} \\
    \cmidrule(lr){1-9}
      \multirow[t]{3}{*}{GPT-5-mini} & Standard & 0.658 & \textbf{0.515} & 0.160 & \textit{0.374} & 0.442 & 0.435 & 0.572 \\
       & Strict & 0.441 & 0.331 & 0.375 & 0.538 & 0.295 & 0.398 & 0.562 \\
       & Persistent & 0.666 & 0.202 & \underline{0.139} & \textbf{0.360} & 0.432 & 0.452 & 0.545 \\
    \cmidrule(lr){1-9}
      \multirow[t]{3}{*}{Claude Sonnet 5} & Standard & 0.443 & 0.415 & 0.247 & 0.488 & 0.534 & 0.503 & \underline{0.626} \\
       & Strict & \textbf{0.385} & 0.488 & \textit{0.155} & 0.391 & \textbf{0.631} & \textit{0.558} & 0.622 \\
       & Persistent & 0.455 & 0.283 & 0.268 & 0.508 & 0.452 & 0.481 & 0.613 \\
    \cmidrule(lr){1-9}
      \multirow[t]{3}{*}{GLM-5.2} & Standard & 0.721 & 0.026 & 0.537 & 0.926 & 0.185 & 0.308 & 0.538 \\
       & Strict & 0.682 & 0.002 & 0.654 & 0.901 & 0.170 & 0.377 & 0.541 \\
       & Persistent & 0.828 & $-0.083$ & 0.732 & 1.041 & 0.214 & 0.379 & 0.545 \\
    \cmidrule(lr){1-9}
      \multirow[t]{3}{*}{Kimi-K2.6} & Standard & 0.727 & 0.046 & 0.508 & 0.854 & 0.150 & 0.415 & 0.526 \\
       & Strict & 0.625 & $-0.059$ & 0.404 & 0.663 & 0.209 & 0.359 & 0.538 \\
       & Persistent & 0.635 & $-0.101$ & 0.524 & 0.757 & 0.179 & 0.352 & 0.544 \\
    \cmidrule(lr){1-9}
      \multirow[t]{3}{*}{GPT-OSS-120B} & Standard & 0.649 & 0.031 & 0.237 & 0.494 & 0.053 & 0.397 & 0.504 \\
       & Strict & 0.521 & 0.044 & \textbf{0.138} & \underline{0.368} & 0.097 & 0.257 & 0.508 \\
       & Persistent & 0.599 & 0.015 & 0.481 & 0.703 & 0.150 & 0.352 & 0.533 \\
    \cmidrule(lr){1-9}
      \multirow[t]{3}{*}{Gemini-3.5-Flash-Lite} & Standard & 1.357 & 0.307 & 0.162 & 0.385 & 0.243 & 0.481 & 0.519 \\
       & Strict & 0.807 & 0.450 & 0.385 & 0.546 & 0.398 & 0.470 & 0.589 \\
       & Persistent & 1.199 & \underline{0.503} & 0.237 & 0.468 & 0.373 & 0.502 & 0.553 \\
    \cmidrule(lr){1-9}
      \multirow[t]{3}{*}{Qwen-3.5-Flash} & Standard & 1.049 & 0.338 & 0.342 & 0.574 & 0.236 & 0.420 & 0.557 \\
       & Strict & 0.571 & 0.405 & 0.291 & 0.540 & 0.191 & 0.398 & 0.524 \\
       & Persistent & 0.905 & 0.420 & 0.382 & 0.594 & 0.595 & 0.516 & 0.587 \\
    \bottomrule
  \end{tabular}
  }
  \endgroup
\end{table}

\begin{table}[!t]
  \centering
  \caption{\textbf{Full Confidence results for secondary scores.} Metrics and column order match \Cref{tab:benchmark_main}. The best result in each column is shown in bold, the second-best is underlined, and the third-best is italicized. A dash denotes an undefined Spearman correlation or ICC.}
  \label{tab:benchmark_confidence_full}
  \begingroup
  \resizebox{\linewidth}{!}{%
  \begin{tabular}{@{}ll|cc|cc|ccc@{}}
    \toprule
    \multirow{2}{*}{\textbf{System}} & \multirow{2}{*}{\textbf{Protocol}} &
    \multicolumn{2}{c}{\textbf{Human alignment}} &
    \multicolumn{2}{c}{\textbf{Within-paper stability}} &
    \multicolumn{3}{c}{\textbf{Joint stability-discrimination}} \\
    \cmidrule(lr){3-4}\cmidrule(lr){5-6}\cmidrule(lr){7-9}
    & & H-MAE $\downarrow$ & Spearman $\uparrow$ &
    MAD $\downarrow$ & Drift SD $\downarrow$ & ICC $\uparrow$ &
    SPR $\uparrow$ & Discrim. $\uparrow$ \\
    \midrule
    \multicolumn{9}{@{}l}{\itshape General-purpose LLM reviewers} \\
    \cmidrule(lr){1-9}
      \multirow[t]{3}{*}{GPT-5.5} & Standard & 0.500 & $-0.013$ & 0.232 & 0.466 & 0.272 & 0.375 & \textit{0.550} \\
       & Strict & 0.516 & $-0.142$ & 0.142 & 0.367 & 0.193 & 0.304 & 0.520 \\
       & Persistent & \textbf{0.466} & 0.139 & 0.093 & 0.304 & 0.263 & \underline{0.450} & 0.517 \\
    \cmidrule(lr){1-9}
      \multirow[t]{3}{*}{GPT-5-mini} & Standard & \underline{0.489} & -- & \textbf{0.000} & \textbf{0.000} & -- & 0.000 & 0.500 \\
       & Strict & 0.505 & $-0.216$ & \underline{0.025} & \underline{0.157} & 0.002 & 0.396 & 0.500 \\
       & Persistent & 0.514 & $-0.102$ & 0.059 & 0.240 & 0.036 & \textit{0.445} & 0.501 \\
    \cmidrule(lr){1-9}
      \multirow[t]{3}{*}{Claude Sonnet 5} & Standard & 0.629 & $-0.154$ & 0.128 & 0.356 & 0.114 & 0.328 & 0.508 \\
       & Strict & 0.626 & $-0.152$ & 0.085 & 0.290 & \textit{0.285} & 0.423 & 0.517 \\
       & Persistent & 0.559 & 0.094 & 0.262 & 0.510 & 0.217 & 0.364 & 0.534 \\
    \cmidrule(lr){1-9}
      \multirow[t]{3}{*}{GLM-5.2} & Standard & 1.143 & $-0.101$ & 0.836 & 1.271 & 0.054 & 0.300 & 0.543 \\
       & Strict & 1.076 & \underline{0.186} & 0.967 & 1.407 & 0.019 & 0.329 & 0.523 \\
       & Persistent & 1.226 & 0.011 & 1.059 & 1.510 & 0.048 & 0.338 & 0.529 \\
    \cmidrule(lr){1-9}
      \multirow[t]{3}{*}{Kimi-K2.6} & Standard & 1.188 & $-0.036$ & 0.765 & 1.136 & 0.145 & 0.342 & \underline{0.552} \\
       & Strict & 1.148 & $-0.011$ & 0.841 & 1.245 & 0.065 & 0.311 & 0.538 \\
       & Persistent & 1.218 & \textbf{0.242} & 0.777 & 1.127 & 0.101 & 0.344 & 0.541 \\
    \cmidrule(lr){1-9}
      \multirow[t]{3}{*}{GPT-OSS-120B} & Standard & 0.570 & 0.110 & 0.483 & 0.739 & 0.055 & 0.334 & 0.508 \\
       & Strict & 0.672 & $-0.119$ & 0.505 & 0.751 & 0.080 & 0.333 & 0.509 \\
       & Persistent & 0.661 & 0.039 & 0.499 & 0.738 & 0.132 & 0.347 & 0.522 \\
    \cmidrule(lr){1-9}
      \multirow[t]{3}{*}{Gemini-3.5-Flash-Lite} & Standard & 1.139 & $-0.197$ & 0.270 & 0.486 & 0.247 & 0.445 & 0.532 \\
       & Strict & 0.922 & $-0.099$ & 0.333 & 0.531 & \underline{0.292} & 0.427 & \textbf{0.567} \\
       & Persistent & 1.132 & $-0.159$ & 0.240 & 0.450 & 0.261 & \textbf{0.467} & 0.528 \\
    \cmidrule(lr){1-9}
      \multirow[t]{3}{*}{Qwen-3.5-Flash} & Standard & 0.592 & $-0.012$ & 0.234 & 0.477 & 0.059 & 0.316 & 0.508 \\
       & Strict & \textit{0.489} & -- & \textit{0.040} & \textit{0.196} & 0.036 & 0.000 & 0.501 \\
       & Persistent & 0.553 & \textit{0.151} & 0.113 & 0.336 & \textbf{0.334} & 0.407 & 0.537 \\
    \bottomrule
  \end{tabular}
  }
  \endgroup
\end{table}

\clearpage

\clearpage
\section{Additional Validation}
\label{app:additional_validation}
\label{app:preservation_diagnostics}

\subsection{Core Technical-Content Fidelity Audit}
\label{app:rewrite_fidelity}

We use DeepSeek V4 Flash as an independent auditor to assess preservation of core
technical content between each rhetorical rewrite and its matched original.
Rhetorical emphasis and evaluative framing are allowed to vary as part of the
intended interventions. The auditor is
not used as a rewrite producer, science-core extractor, or reviewer
elsewhere in our experiments. The audit examines five dimensions:
\emph{table fidelity}, covering table structure, labels, values, and
method--dataset--metric associations; \emph{numerical fidelity}, covering
numerical values, statistics, confidence intervals, and p-values in the
manuscript body; \emph{experimental fidelity}, covering datasets, evaluated
subsets, splits, sample sizes, metrics, baselines, and experimental settings;
\emph{method fidelity}, covering method components, algorithmic operations,
equations, variables, assumptions, and technical dependencies; and
\emph{result fidelity}, covering the direction and ordering of empirical
results and positive versus negative observed effects.

The audit uses a binary decision rule. For paper $i$, rhetorical variant $k$,
and applicable fidelity dimension $d$, let $z_{ik}^{(d)}=1$ when no concrete
technical-content mismatch is identified and $z_{ik}^{(d)}=0$ when the
auditor identifies a mismatch supported by evidence from the matched
manuscripts. The fidelity rate for dimension $d$ is
\[
\operatorname{Fidelity}_{d}
=
\frac{1}{N_d}
\sum_{(i,k)\in\mathcal{A}_d} z_{ik}^{(d)},
\qquad
N_d=\lvert\mathcal{A}_d\rvert,
\]
where $\mathcal{A}_d$ contains the comparisons to which dimension $d$
applies.

To complement the automated audit, we conduct a human validation on 150
comparisons between originals and variants. We first randomly sample 120 of the 1,200
rhetorical variants (10\%), each evaluated against its matched original, and
supplement them with 30 additional comparisons for which the automated auditor
identifies at least one fidelity mismatch. Three graduate-level annotators
independently assess all 150 comparisons using the same five
dimensions: table, numerical, experimental, method, and result fidelity. For each comparison and dimension, the
human verdict is determined by majority vote among the three annotators.
Dimension-level human pass rates summarize these consensus verdicts across
the combined 150-comparison audit sample.
Overall inter-annotator agreement, measured as the mean pairwise
agreement across all 750 dimension-level judgments, is 94.6\%.

\begin{center}
\centering
\captionof{table}{\textbf{Core technical-content fidelity of rhetorical rewrites.}
Dimension-level rates report the proportion of applicable comparisons that
pass each check. Human rates summarize the combined audit sample of 120
randomly selected comparisons and 30 additional comparisons flagged by
the automated auditor.}
\label{tab:rewrite_fidelity_audit}
\begingroup
\small
\setlength{\tabcolsep}{5.5pt}
\renewcommand{\arraystretch}{1.08}
\begin{tabular*}{\textwidth}{@{\extracolsep{\fill}}lccccc@{}}
\toprule
& \multicolumn{5}{c}{\textbf{Dimension-level fidelity}} \\
\cmidrule(lr){2-6}
\textbf{Metric} & Table & Numerical & Experimental & Method & Result \\
\midrule
Automated pass rate (\%) & 95.2 & 96.4 & 94.8 & 97.3 & 96.5 \\
\midrule
Human pass rate (\%) & 97.3 & 98.0 & 96.7 & 98.7 & 97.3 \\
\bottomrule
\end{tabular*}
\endgroup
\end{center}

Across the five assessed dimensions, the automated and human checks indicate
that core technical content is largely preserved, with occasional
discrepancies. The audit serves as a diagnostic assessment, and all
comparisons are retained in the reported evaluation. Residual content
differences may contribute to observed score variation.

\subsection{Complementary Review Branches under Equal-Weight Fusion}
\label{app:manuscript_ensemble_controls}

Table~\ref{tab:manuscript_ensemble_controls} compares \oursm{} with
full-manuscript score averages. We reuse cached scores, average them
without rounding, and recompute all seven metrics. The Strict-anchored
comparisons reuse the same manuscript-branch scores.

Some apparent robustness gains may arise from averaging itself: averaging
integer scores allows half points, can reduce MAD and Drift SD through
drift cancellation, and changes distance ties in discriminability. To
assess whether \oursm{} retains an advantage under the same averaging
operation, the two-score manuscript controls match its averaging rule and
score resolution. The three-protocol average additionally matches its
nominal three-call budget: three reviews versus one extraction and two
reviews, with potentially different token costs.

\begin{table}[!htbp]
\centering
\caption{\textbf{Full-manuscript averaging controls on \oursb{}.} All configurations use GPT-5.5 and the same 60 papers with 21 versions each. Protocol names denote full-manuscript reviews. Scores are averaged separately for each original and rewrite before computing the metrics, with no rounding. \oursm{} averages Manuscript-Strict and Core-Adapted. The three-protocol control averages three review scores. All entries are point estimates.}
\label{tab:manuscript_ensemble_controls}
\begingroup
\footnotesize
\setlength{\tabcolsep}{3pt}
\renewcommand{\arraystretch}{1.05}
\resizebox{\textwidth}{!}{%
\begin{tabular}{@{}l|cc|cc|ccc@{}}
\toprule
\multirow{2}{*}{\textbf{Configuration}} &
\multicolumn{2}{c}{\textbf{Human alignment}} &
\multicolumn{2}{c}{\textbf{Within-paper stability}} &
\multicolumn{3}{c}{\textbf{Joint stability-discrimination}} \\
\cmidrule(lr){2-3}\cmidrule(lr){4-5}\cmidrule(lr){6-8}
& H-MAE $\downarrow$ & Spearman $\uparrow$ & MAD $\downarrow$ & Drift SD $\downarrow$ & ICC $\uparrow$ & SPR $\uparrow$ & Discrim. $\uparrow$ \\
\midrule
\multicolumn{8}{@{}l}{\itshape Single full-manuscript reviews} \\
\cmidrule(lr){1-8}
Manuscript-Standard & 1.294 & 0.448 & 0.598 & 0.956 & 0.615 & 0.555 & 0.649 \\
Manuscript-Strict & 1.078 & 0.529 & 0.766 & 1.202 & 0.632 & 0.507 & 0.672 \\
Manuscript-Persistent & 1.261 & 0.423 & 0.476 & 0.875 & 0.697 & 0.586 & 0.682 \\
\midrule
\multicolumn{8}{@{}l}{\itshape Two-score averages} \\
\cmidrule(lr){1-8}
Standard + Strict & 1.130 & 0.495 & 0.560 & 0.820 & 0.750 & 0.620 & 0.705 \\
Standard + Persistent & 1.180 & 0.455 & 0.470 & 0.681 & 0.755 & 0.640 & 0.712 \\
Strict + Persistent & 1.066 & 0.494 & 0.550 & 0.830 & 0.758 & 0.635 & 0.715 \\
\textbf{\oursm{}} & 1.072 & 0.488 & 0.476 & 0.687 & 0.775 & 0.652 & 0.726 \\
\midrule
\multicolumn{8}{@{}l}{\itshape Three-score average} \\
\cmidrule(lr){1-8}
Standard + Strict + Persistent & 1.110 & 0.466 & 0.520 & 0.681 & 0.781 & 0.635 & 0.715 \\
\bottomrule
\end{tabular}%
}
\endgroup
\end{table}

With Manuscript-Strict and equal weighting fixed, Core-Adapted improves
all five robustness metrics over either manuscript alternative, supported
by paired bootstrap intervals (Appendix~\ref{app:paper_family_bootstrap}).
Thus, the robustness advantage over these two controls persists when the
manuscript branch, averaging rule, and score resolution are matched.
Both alternatives have higher Spearman correlation, and Strict + Persistent
also has lower H-MAE. Standard + Persistent outperforms \oursm{} on MAD
and Drift SD, while the three-protocol average outperforms it on Drift SD
and ICC; \oursm{} has better point estimates on the other five metrics
in each comparison and remains Pareto non-dominated.

\subsection{Repeated-Review Noise Baseline}
\label{app:repeated_review_noise}
Following the repeated-review procedure in Appendix~\ref{app:reviewer_execution},
we obtain three independent reviews of each identical PDF, holding the
review prompts and inference settings fixed. We apply this procedure to
the three manuscript protocols and \oursm{} with GPT-5.5, following each
configuration's evaluation pipeline. MAD and ICC quantify variation across
repeated reviews of the same manuscript using the definitions in
Appendix~\ref{app:metric_definitions}.

Repeated reviews of identical PDFs yield MAD of 0.294--0.329 and ICC of
0.846--0.862 across the three manuscript protocols and \oursm{}.
\oursm{} and Manuscript-Strict have similar repeatability (MAD:
0.320 versus 0.329; ICC: 0.858 versus 0.851), but differ more under rewriting
(MAD: 0.476 versus 0.766; ICC: 0.775 versus 0.632), supporting a robustness gain beyond
the small difference in same-PDF repeatability.

\subsection{Paper-Family Bootstrap Analysis}
\label{app:paper_family_bootstrap}

Tables~\ref{tab:bootstrap_paired_comparisons} and
\ref{tab:bootstrap_main_intervals} use 5,000 paired bootstrap resamples of
the 60 paper families, keeping each original and its 20 variants together.
The same resamples are used across configurations to recompute metrics
and paired differences. Marginal 95\% intervals use the 2.5th and 97.5th
percentiles, with the recorded reviews held fixed.

The paired intervals in Table~\ref{tab:bootstrap_paired_comparisons} support
\oursm{}'s gains on all five robustness metrics over Manuscript-Strict and
both two-score manuscript ensembles, with tradeoffs in human alignment.
Relative to Core-Adapted, fusion improves human alignment, ICC, and
discriminability, while increasing MAD and Drift SD and reducing SPR.
Table~\ref{tab:bootstrap_main_intervals} gives metric intervals for the
primary comparison.

\begin{table}[!htbp]
\centering
\caption{\textbf{Paired bootstrap comparisons for \oursm{}.} All configurations use GPT-5.5 and the same 60 complete paper families. Each cell shows the difference (\oursm{} minus the column configuration), followed by its marginal 95\% percentile interval from 5,000 paired paper-family resamples. Negative differences favor \oursm{} for H-MAE, MAD, and Drift SD; positive differences favor \oursm{} for the other metrics.}
\label{tab:bootstrap_paired_comparisons}
\begingroup
\footnotesize
\setlength{\tabcolsep}{4pt}
\renewcommand{\arraystretch}{1.25}
\begin{tabular*}{\textwidth}{@{\extracolsep{\fill}}lcccc@{}}
\toprule
\textbf{Metric} & \shortstack{Manuscript-\\Strict} & Core-Adapted & \shortstack{Strict +\\Standard} & \shortstack{Strict +\\Persistent} \\
\midrule
H-MAE & \shortstack{$-0.006$\\$[-0.011,\,-0.001]$} & \shortstack{$-0.389$\\$[-0.564,\,-0.214]$} & \shortstack{$-0.058$\\$[-0.087,\,-0.029]$} & \shortstack{$+0.006$\\$[+0.001,\,+0.011]$} \\
Spearman & \shortstack{$-0.041$\\$[-0.068,\,-0.014]$} & \shortstack{$+0.203$\\$[+0.053,\,+0.366]$} & \shortstack{$-0.007$\\$[-0.012,\,-0.002]$} & \shortstack{$-0.006$\\$[-0.011,\,-0.001]$} \\
\midrule
MAD & \shortstack{$-0.290$\\$[-0.348,\,-0.232]$} & \shortstack{$+0.184$\\$[+0.084,\,+0.285]$} & \shortstack{$-0.084$\\$[-0.112,\,-0.056]$} & \shortstack{$-0.074$\\$[-0.101,\,-0.047]$} \\
Drift SD & \shortstack{$-0.515$\\$[-0.575,\,-0.455]$} & \shortstack{$+0.037$\\$[+0.006,\,+0.071]$} & \shortstack{$-0.133$\\$[-0.167,\,-0.099]$} & \shortstack{$-0.143$\\$[-0.180,\,-0.106]$} \\
ICC & \shortstack{$+0.143$\\$[+0.101,\,+0.185]$} & \shortstack{$+0.024$\\$[+0.004,\,+0.049]$} & \shortstack{$+0.025$\\$[+0.010,\,+0.040]$} & \shortstack{$+0.017$\\$[+0.005,\,+0.029]$} \\
SPR & \shortstack{$+0.145$\\$[+0.101,\,+0.189]$} & \shortstack{$-0.058$\\$[-0.101,\,-0.012]$} & \shortstack{$+0.032$\\$[+0.015,\,+0.049]$} & \shortstack{$+0.017$\\$[+0.005,\,+0.029]$} \\
Discrim. & \shortstack{$+0.054$\\$[+0.036,\,+0.072]$} & \shortstack{$+0.031$\\$[+0.001,\,+0.058]$} & \shortstack{$+0.021$\\$[+0.010,\,+0.032]$} & \shortstack{$+0.011$\\$[+0.003,\,+0.019]$} \\
\bottomrule
\end{tabular*}
\endgroup
\end{table}

\begin{table}[!htbp]
\centering
\caption{\textbf{Paper-family bootstrap intervals for the primary comparison.} Entries are marginal 95\% percentile intervals from 5,000 paired resamples of the 60 complete paper families; the corresponding point estimates appear in Table~\ref{tab:benchmark_main}. All original and rewritten versions stay together within each resample.}
\label{tab:bootstrap_main_intervals}
\begingroup
\footnotesize
\setlength{\tabcolsep}{2.2pt}
\renewcommand{\arraystretch}{1.12}
\resizebox{\textwidth}{!}{%
\begin{tabular}{@{}ll|cc|cc|ccc@{}}
\toprule
\multirow{2}{*}{\textbf{System}} & \multirow{2}{*}{\textbf{Protocol}} &
\multicolumn{2}{c}{\textbf{Human alignment}} &
\multicolumn{2}{c}{\textbf{Within-paper stability}} &
\multicolumn{3}{c}{\textbf{Joint stability-discrimination}} \\
\cmidrule(lr){3-4}\cmidrule(lr){5-6}\cmidrule(lr){7-9}
& & H-MAE $\downarrow$ & Spearman $\uparrow$ & MAD $\downarrow$ & Drift SD $\downarrow$ & ICC $\uparrow$ & SPR $\uparrow$ & Discrim. $\uparrow$ \\
\midrule
\multirow[t]{3}{*}{GPT-5.5} & Standard & $[1.067,\,1.531]$ & $[0.213,\,0.640]$ & $[0.476,\,0.717]$ & $[0.809,\,1.075]$ & $[0.492,\,0.698]$ & $[0.441,\,0.630]$ & $[0.614,\,0.679]$ \\
 & Strict & $[0.884,\,1.287]$ & $[0.305,\,0.700]$ & $[0.603,\,0.927]$ & $[1.038,\,1.338]$ & $[0.520,\,0.709]$ & $[0.416,\,0.593]$ & $[0.632,\,0.702]$ \\
 & Persistent & $[1.037,\,1.491]$ & $[0.161,\,0.638]$ & $[0.364,\,0.593]$ & $[0.715,\,1.017]$ & $[0.573,\,0.770]$ & $[0.430,\,0.696]$ & $[0.640,\,0.716]$ \\
\cmidrule(lr){1-9}
\multirow[t]{3}{*}{GPT-5-mini} & Standard & $[1.387,\,2.048]$ & $[0.148,\,0.634]$ & $[0.431,\,0.664]$ & $[0.821,\,1.057]$ & $[0.276,\,0.504]$ & $[0.263,\,0.487]$ & $[0.553,\,0.616]$ \\
 & Strict & $[1.014,\,1.471]$ & $[0.142,\,0.616]$ & $[0.749,\,1.046]$ & $[1.111,\,1.306]$ & $[0.283,\,0.481]$ & $[0.352,\,0.487]$ & $[0.567,\,0.625]$ \\
 & Persistent & $[1.344,\,1.919]$ & $[0.098,\,0.542]$ & $[0.426,\,0.749]$ & $[0.792,\,1.152]$ & $[0.326,\,0.581]$ & $[0.357,\,0.555]$ & $[0.551,\,0.623]$ \\
\cmidrule(lr){1-9}
\multirow[t]{3}{*}{Claude Sonnet 5} & Standard & $[0.962,\,1.432]$ & $[0.137,\,0.649]$ & $[0.364,\,0.577]$ & $[0.730,\,0.972]$ & $[0.467,\,0.671]$ & $[0.439,\,0.636]$ & $[0.607,\,0.678]$ \\
 & Strict & $[0.944,\,1.372]$ & $[0.214,\,0.679]$ & $[0.491,\,0.740]$ & $[0.919,\,1.157]$ & $[0.422,\,0.641]$ & $[0.457,\,0.627]$ & $[0.601,\,0.668]$ \\
 & Persistent & $[1.027,\,1.473]$ & $[0.053,\,0.602]$ & $[0.218,\,0.414]$ & $[0.571,\,0.834]$ & $[0.378,\,0.568]$ & $[0.350,\,0.575]$ & $[0.573,\,0.648]$ \\
\cmidrule(lr){1-9}
\multirow[t]{3}{*}{GLM-5.2} & Standard & $[1.706,\,2.394]$ & $[-0.296,\,0.241]$ & $[1.264,\,1.744]$ & $[2.071,\,2.572]$ & $[0.143,\,0.305]$ & $[0.196,\,0.411]$ & $[0.539,\,0.576]$ \\
 & Strict & $[1.687,\,2.393]$ & $[-0.294,\,0.222]$ & $[1.391,\,1.709]$ & $[1.898,\,2.250]$ & $[0.139,\,0.283]$ & $[0.337,\,0.429]$ & $[0.538,\,0.572]$ \\
 & Persistent & $[2.002,\,2.822]$ & $[-0.451,\,0.094]$ & $[1.663,\,2.218]$ & $[2.407,\,2.902]$ & $[0.147,\,0.308]$ & $[0.328,\,0.432]$ & $[0.536,\,0.579]$ \\
\cmidrule(lr){1-9}
\multirow[t]{3}{*}{Kimi-K2.6} & Standard & $[1.374,\,2.025]$ & $[0.027,\,0.504]$ & $[1.128,\,1.595]$ & $[1.698,\,2.237]$ & $[0.131,\,0.291]$ & $[0.357,\,0.474]$ & $[0.535,\,0.577]$ \\
 & Strict & $[1.419,\,2.072]$ & $[-0.200,\,0.300]$ & $[0.610,\,0.912]$ & $[1.093,\,1.424]$ & $[0.135,\,0.431]$ & $[0.280,\,0.447]$ & $[0.522,\,0.569]$ \\
 & Persistent & $[1.261,\,1.923]$ & $[-0.233,\,0.329]$ & $[0.987,\,1.319]$ & $[1.453,\,1.824]$ & $[0.161,\,0.410]$ & $[0.347,\,0.426]$ & $[0.543,\,0.596]$ \\
\cmidrule(lr){1-9}
\multirow[t]{3}{*}{GPT-OSS-120B} & Standard & $[1.298,\,1.798]$ & $[-0.195,\,0.336]$ & $[0.586,\,0.863]$ & $[0.940,\,1.297]$ & $[0.089,\,0.186]$ & $[0.277,\,0.443]$ & $[0.521,\,0.541]$ \\
 & Strict & $[1.379,\,2.024]$ & $[-0.163,\,0.375]$ & $[0.276,\,0.511]$ & $[0.733,\,1.013]$ & $[0.049,\,0.132]$ & $[0.176,\,0.428]$ & $[0.504,\,0.513]$ \\
 & Persistent & $[1.289,\,1.774]$ & $[-0.250,\,0.237]$ & $[0.707,\,1.043]$ & $[1.096,\,1.545]$ & $[0.112,\,0.207]$ & $[0.266,\,0.392]$ & $[0.526,\,0.544]$ \\
\cmidrule(lr){1-9}
\multirow[t]{3}{*}{Gemini 3.5 Flash-Lite} & Standard & $[2.828,\,3.532]$ & $[0.208,\,0.556]$ & $[0.105,\,0.323]$ & $[0.450,\,0.773]$ & $[0.074,\,0.311]$ & $[0.340,\,0.521]$ & $[0.503,\,0.523]$ \\
 & Strict & $[1.553,\,2.209]$ & $[0.242,\,0.654]$ & $[0.691,\,0.998]$ & $[1.071,\,1.296]$ & $[0.316,\,0.512]$ & $[0.431,\,0.587]$ & $[0.573,\,0.622]$ \\
 & Persistent & $[2.491,\,3.176]$ & $[0.271,\,0.649]$ & $[0.268,\,0.651]$ & $[0.715,\,1.218]$ & $[0.217,\,0.449]$ & $[0.446,\,0.540]$ & $[0.518,\,0.561]$ \\
\cmidrule(lr){1-9}
\multirow[t]{3}{*}{Qwen 3.5 Flash} & Standard & $[1.988,\,2.722]$ & $[0.129,\,0.595]$ & $[0.645,\,1.019]$ & $[1.105,\,1.628]$ & $[0.187,\,0.388]$ & $[0.396,\,0.516]$ & $[0.538,\,0.584]$ \\
 & Strict & $[1.046,\,1.433]$ & $[0.180,\,0.586]$ & $[0.739,\,1.009]$ & $[1.177,\,1.487]$ & $[0.133,\,0.311]$ & $[0.304,\,0.468]$ & $[0.524,\,0.558]$ \\
 & Persistent & $[1.699,\,2.393]$ & $[0.244,\,0.635]$ & $[0.752,\,0.975]$ & $[1.167,\,1.372]$ & $[0.228,\,0.752]$ & $[0.369,\,0.640]$ & $[0.553,\,0.642]$ \\
\cmidrule(lr){1-9}
OpenReviewer & -- & $[1.126,\,1.708]$ & $[0.096,\,0.550]$ & $[0.929,\,1.251]$ & $[1.328,\,1.703]$ & $[0.140,\,0.279]$ & $[0.300,\,0.424]$ & $[0.527,\,0.556]$ \\
CycleReviewer & -- & $[1.159,\,1.660]$ & $[-0.172,\,0.310]$ & $[0.663,\,0.913]$ & $[0.895,\,1.178]$ & $[0.081,\,0.176]$ & $[0.292,\,0.389]$ & $[0.521,\,0.545]$ \\
DeepReviewer & -- & $[1.204,\,1.749]$ & $[0.178,\,0.587]$ & $[0.362,\,0.553]$ & $[0.583,\,0.784]$ & $[0.087,\,0.229]$ & $[0.293,\,0.433]$ & $[0.523,\,0.549]$ \\
AI Scientist & -- & $[1.219,\,1.822]$ & $[-0.171,\,0.347]$ & $[0.698,\,0.989]$ & $[1.027,\,1.471]$ & $[0.142,\,0.273]$ & $[0.240,\,0.461]$ & $[0.529,\,0.559]$ \\
OpenJudge & -- & $[1.112,\,1.504]$ & $[-0.131,\,0.362]$ & $[0.244,\,0.434]$ & $[0.501,\,0.678]$ & $[0.107,\,0.345]$ & $[0.338,\,0.512]$ & $[0.509,\,0.548]$ \\
ProReviewer & -- & $[1.260,\,1.697]$ & $[-0.017,\,0.463]$ & $[0.632,\,0.796]$ & $[0.882,\,1.057]$ & $[0.058,\,0.128]$ & $[0.275,\,0.385]$ & $[0.509,\,0.524]$ \\
\midrule
\textbf{\oursm{}} & -- & $[0.878,\,1.278]$ & $[0.248,\,0.678]$ & $[0.394,\,0.558]$ & $[0.600,\,0.759]$ & $[0.667,\,0.835]$ & $[0.545,\,0.731]$ & $[0.678,\,0.760]$ \\
\bottomrule
\end{tabular}%
}
\endgroup
\end{table}

\subsection{Science-Core Preservation Diagnostics}

The preservation analysis compares the embedding of a science-core report
extracted from an original manuscript with the embedding of the report
extracted from its matched rhetorical variant:
\[
\operatorname{cos}\!\left(
E(r_i^{\mathrm{original}}),
E(r_{ikp}^{\mathrm{variant}})
\right).
\]
Here \(i\) indexes papers, \(k\) rewrite conditions, and \(p\) rewrite
producers. The condition-level tables keep GPT-5.5 and Opus 4.8 rewrite
producers separate and report the mean, median, and fifth percentile.

We embed reports with \texttt{text-embedding-3-small} via the OpenAI API using
default API parameters.

\label{app:condition_producer}

\begin{table}[!h]
\centering
\caption{\textbf{Condition-level science-core similarity for
GPT-5-mini.} P05 denotes the fifth percentile across comparisons. The
cross-paper baseline compares variant cores with nonmatching original cores.}
\label{tab:report_similarity_conditions_gptmini}
\begingroup
\footnotesize
\setlength{\tabcolsep}{4.5pt}
\renewcommand{\arraystretch}{0.96}
\begin{tabular}{@{}lcccccc@{}}
\toprule
\multirow{2}{*}{\textbf{Rhetorical condition}} &
\multicolumn{3}{c}{\textbf{GPT-5.5 producer}} &
\multicolumn{3}{c}{\textbf{Opus 4.8 producer}} \\
\cmidrule(lr){2-4}\cmidrule(lr){5-7}
& Mean & Median & P05 & Mean & Median & P05 \\
\midrule
Novelty stance & 0.975 & 0.977 & 0.963 & 0.965 & 0.976 & 0.960 \\
Scope framing & 0.976 & 0.976 & 0.969 & 0.976 & 0.977 & 0.965 \\
Evidence framing & 0.972 & 0.974 & 0.957 & 0.973 & 0.975 & 0.958 \\
Contribution salience & 0.974 & 0.976 & 0.964 & 0.975 & 0.976 & 0.963 \\
Technical register & 0.976 & 0.976 & 0.968 & 0.974 & 0.975 & 0.964 \\
Linguistic complexity & 0.966 & 0.978 & 0.966 & 0.977 & 0.977 & 0.968 \\
Joint rewrite & 0.977 & 0.978 & 0.969 & 0.975 & 0.977 & 0.962 \\
Recursive rewrite (R2) & 0.977 & 0.978 & 0.967 & 0.974 & 0.974 & 0.962 \\
Recursive rewrite (R3) & 0.975 & 0.978 & 0.963 & 0.974 & 0.975 & 0.963 \\
Reviewer-guided rewrite & 0.976 & 0.977 & 0.964 & 0.972 & 0.974 & 0.958 \\
\midrule
Cross-paper baseline & 0.668 & 0.668 & 0.582 & 0.670 & 0.671 & 0.585 \\
\bottomrule
\end{tabular}
\endgroup
\end{table}

\begin{table}[!h]
\centering
\caption{\textbf{Condition-level science-core similarity for GPT-5.5.}
P05 denotes the fifth percentile across comparisons. The cross-paper baseline
compares variant cores with nonmatching original cores.}
\label{tab:report_similarity_conditions_gpt55}
\begingroup
\footnotesize
\setlength{\tabcolsep}{4.5pt}
\renewcommand{\arraystretch}{0.96}
\begin{tabular}{@{}lcccccc@{}}
\toprule
\multirow{2}{*}{\textbf{Rhetorical condition}} &
\multicolumn{3}{c}{\textbf{GPT-5.5 producer}} &
\multicolumn{3}{c}{\textbf{Opus 4.8 producer}} \\
\cmidrule(lr){2-4}\cmidrule(lr){5-7}
& Mean & Median & P05 & Mean & Median & P05 \\
\midrule
Novelty stance & 0.979 & 0.979 & 0.968 & 0.979 & 0.981 & 0.967 \\
Scope framing & 0.978 & 0.978 & 0.965 & 0.980 & 0.980 & 0.969 \\
Evidence framing & 0.976 & 0.978 & 0.964 & 0.977 & 0.979 & 0.961 \\
Contribution salience & 0.976 & 0.978 & 0.963 & 0.977 & 0.979 & 0.965 \\
Technical register & 0.978 & 0.978 & 0.968 & 0.976 & 0.977 & 0.966 \\
Linguistic complexity & 0.979 & 0.980 & 0.967 & 0.980 & 0.981 & 0.970 \\
Joint rewrite & 0.980 & 0.981 & 0.970 & 0.979 & 0.979 & 0.968 \\
Recursive rewrite (R2) & 0.979 & 0.980 & 0.970 & 0.978 & 0.979 & 0.965 \\
Recursive rewrite (R3) & 0.979 & 0.979 & 0.970 & 0.977 & 0.978 & 0.962 \\
Reviewer-guided rewrite & 0.977 & 0.978 & 0.962 & 0.977 & 0.978 & 0.966 \\
\midrule
Cross-paper baseline & 0.616 & 0.617 & 0.516 & 0.619 & 0.620 & 0.518 \\
\bottomrule
\end{tabular}
\endgroup
\end{table}

\begin{table}[!h]
\centering
\caption{\textbf{Condition-level science-core similarity for
Gemini-3.5-Flash-Lite.} P05 denotes the fifth percentile across comparisons.
The cross-paper baseline compares variant cores with nonmatching original
cores.}
\label{tab:report_similarity_conditions_gemini35}
\begingroup
\footnotesize
\setlength{\tabcolsep}{4.5pt}
\renewcommand{\arraystretch}{0.96}
\begin{tabular}{@{}lcccccc@{}}
\toprule
\multirow{2}{*}{\textbf{Rhetorical condition}} &
\multicolumn{3}{c}{\textbf{GPT-5.5 producer}} &
\multicolumn{3}{c}{\textbf{Opus 4.8 producer}} \\
\cmidrule(lr){2-4}\cmidrule(lr){5-7}
& Mean & Median & P05 & Mean & Median & P05 \\
\midrule
Novelty stance & 0.963 & 0.967 & 0.939 & 0.965 & 0.965 & 0.940 \\
Scope framing & 0.961 & 0.965 & 0.927 & 0.964 & 0.967 & 0.944 \\
Evidence framing & 0.958 & 0.961 & 0.930 & 0.961 & 0.963 & 0.935 \\
Contribution salience & 0.959 & 0.964 & 0.934 & 0.961 & 0.963 & 0.930 \\
Technical register & 0.962 & 0.965 & 0.940 & 0.961 & 0.964 & 0.935 \\
Linguistic complexity & 0.964 & 0.969 & 0.934 & 0.963 & 0.965 & 0.939 \\
Joint rewrite & 0.964 & 0.967 & 0.930 & 0.959 & 0.963 & 0.937 \\
Recursive rewrite (R2) & 0.964 & 0.966 & 0.933 & 0.959 & 0.963 & 0.928 \\
Recursive rewrite (R3) & 0.966 & 0.968 & 0.937 & 0.961 & 0.960 & 0.942 \\
Reviewer-guided rewrite & 0.963 & 0.965 & 0.937 & 0.959 & 0.961 & 0.930 \\
\midrule
Cross-paper baseline & 0.606 & 0.606 & 0.506 & 0.609 & 0.608 & 0.509 \\
\bottomrule
\end{tabular}
\endgroup
\end{table}

\subsection{Complete Science-Core Weight Sensitivity Results}
\label{app:scicore_weight_sensitivity}

Figure~\ref{fig:scicore_fusion_weight_ablation} provides a descriptive,
direction-normalized visualization of the science-core weight sweep. The
complete unnormalized numerical results are reported in
Tables~\ref{tab:scicore_weight_standard}--\ref{tab:scicore_weight_persistent}.
Each table fixes the manuscript protocol and varies the science-core weight
$\alpha$ from 0 (manuscript only) to 1 (science core only). Intermediate rows
use the continuous post-hoc fusion in Equation~\ref{eq:scicore_weight_sweep}.

\begin{table}[!htbp]
  \centering
  \caption{\textbf{Complete science-core weight sensitivity with Manuscript-Standard.} All entries are unnormalized numerical results. Arrows indicate the preferred direction.}
  \label{tab:scicore_weight_standard}
  \begingroup
  \scriptsize
  \setlength{\tabcolsep}{5pt}
  \renewcommand{\arraystretch}{1.02}
  \begin{tabular}{@{}c|cc|ccc|cc@{}}
    \toprule
    \multirow{2}{*}{$\boldsymbol{\alpha}$} &
    \multicolumn{2}{c}{\textbf{Within-paper stability}} &
    \multicolumn{3}{c}{\textbf{Joint stability-discrimination}} &
    \multicolumn{2}{c}{\textbf{Human alignment}} \\
    \cmidrule(lr){2-3}\cmidrule(lr){4-6}\cmidrule(lr){7-8}
    & MAD $\downarrow$ & Drift SD $\downarrow$ & ICC $\uparrow$ &
    SPR $\uparrow$ & Discrim. $\uparrow$ & H-MAE $\downarrow$ & Spearman $\uparrow$ \\
    \midrule
    0.00 & 0.598 & 0.956 & 0.615 & 0.555 & 0.649 & 1.294 & 0.448 \\
    0.05 & 0.577 & 0.911 & 0.633 & 0.566 & 0.699 & 1.294 & 0.441 \\
    0.10 & 0.556 & 0.867 & 0.652 & 0.577 & 0.699 & 1.293 & 0.441 \\
    0.15 & 0.534 & 0.824 & 0.671 & 0.590 & 0.699 & 1.292 & 0.441 \\
    0.20 & 0.513 & 0.784 & 0.690 & 0.603 & 0.701 & 1.291 & 0.441 \\
    0.25 & 0.492 & 0.745 & 0.708 & 0.618 & 0.705 & 1.290 & 0.434 \\
    0.30 & 0.470 & 0.708 & 0.726 & 0.633 & 0.711 & 1.293 & 0.425 \\
    0.35 & 0.450 & 0.675 & 0.743 & 0.648 & 0.716 & 1.296 & 0.406 \\
    0.40 & 0.430 & 0.644 & 0.758 & 0.664 & 0.717 & 1.300 & 0.406 \\
    0.45 & 0.410 & 0.617 & 0.772 & 0.679 & 0.719 & 1.304 & 0.406 \\
    0.50 & 0.390 & 0.594 & 0.783 & 0.693 & 0.723 & 1.308 & 0.396 \\
    0.55 & 0.380 & 0.576 & 0.791 & 0.705 & 0.739 & 1.321 & 0.375 \\
    0.60 & 0.369 & 0.562 & 0.797 & 0.716 & 0.738 & 1.333 & 0.369 \\
    0.65 & 0.359 & 0.555 & 0.800 & 0.724 & 0.738 & 1.346 & 0.369 \\
    0.70 & 0.350 & 0.552 & 0.800 & 0.730 & 0.742 & 1.358 & 0.369 \\
    0.75 & 0.340 & 0.556 & 0.797 & 0.732 & 0.739 & 1.371 & 0.369 \\
    0.80 & 0.330 & 0.565 & 0.792 & 0.732 & 0.737 & 1.387 & 0.369 \\
    0.85 & 0.321 & 0.579 & 0.784 & 0.729 & 0.736 & 1.404 & 0.369 \\
    0.90 & 0.311 & 0.599 & 0.775 & 0.724 & 0.735 & 1.423 & 0.369 \\
    0.95 & 0.301 & 0.622 & 0.763 & 0.718 & 0.735 & 1.442 & 0.369 \\
    1.00 & 0.292 & 0.650 & 0.751 & 0.710 & 0.695 & 1.461 & 0.285 \\
    \bottomrule
  \end{tabular}
  \endgroup
\end{table}

\begin{table}[!htbp]
  \centering
  \caption{\textbf{Complete science-core weight sensitivity with Manuscript-Strict.} All entries are unnormalized numerical results. Arrows indicate the preferred direction.}
  \label{tab:scicore_weight_strict}
  \begingroup
  \scriptsize
  \setlength{\tabcolsep}{5pt}
  \renewcommand{\arraystretch}{1.02}
  \begin{tabular}{@{}c|cc|ccc|cc@{}}
    \toprule
    \multirow{2}{*}{$\boldsymbol{\alpha}$} &
    \multicolumn{2}{c}{\textbf{Within-paper stability}} &
    \multicolumn{3}{c}{\textbf{Joint stability-discrimination}} &
    \multicolumn{2}{c}{\textbf{Human alignment}} \\
    \cmidrule(lr){2-3}\cmidrule(lr){4-6}\cmidrule(lr){7-8}
    & MAD $\downarrow$ & Drift SD $\downarrow$ & ICC $\uparrow$ &
    SPR $\uparrow$ & Discrim. $\uparrow$ & H-MAE $\downarrow$ & Spearman $\uparrow$ \\
    \midrule
    0.00 & 0.766 & 1.202 & 0.632 & 0.507 & 0.672 & 1.078 & 0.529 \\
    0.05 & 0.736 & 1.143 & 0.645 & 0.516 & 0.712 & 1.061 & 0.513 \\
    0.10 & 0.707 & 1.085 & 0.659 & 0.526 & 0.712 & 1.044 & 0.513 \\
    0.15 & 0.677 & 1.028 & 0.674 & 0.538 & 0.712 & 1.032 & 0.513 \\
    0.20 & 0.648 & 0.972 & 0.689 & 0.550 & 0.713 & 1.022 & 0.513 \\
    0.25 & 0.619 & 0.918 & 0.704 & 0.565 & 0.715 & 1.014 & 0.510 \\
    0.30 & 0.589 & 0.866 & 0.720 & 0.580 & 0.719 & 1.016 & 0.504 \\
    0.35 & 0.560 & 0.817 & 0.735 & 0.597 & 0.722 & 1.022 & 0.497 \\
    0.40 & 0.532 & 0.770 & 0.749 & 0.615 & 0.723 & 1.039 & 0.497 \\
    0.45 & 0.504 & 0.727 & 0.763 & 0.633 & 0.725 & 1.056 & 0.497 \\
    0.50 & 0.476 & 0.687 & 0.775 & 0.652 & 0.726 & 1.072 & 0.488 \\
    0.55 & 0.456 & 0.653 & 0.786 & 0.671 & 0.735 & 1.107 & 0.474 \\
    0.60 & 0.435 & 0.623 & 0.794 & 0.688 & 0.734 & 1.142 & 0.451 \\
    0.65 & 0.414 & 0.600 & 0.800 & 0.703 & 0.733 & 1.177 & 0.451 \\
    0.70 & 0.396 & 0.585 & 0.802 & 0.716 & 0.745 & 1.212 & 0.415 \\
    0.75 & 0.378 & 0.576 & 0.801 & 0.724 & 0.741 & 1.247 & 0.406 \\
    0.80 & 0.361 & 0.576 & 0.797 & 0.729 & 0.742 & 1.286 & 0.397 \\
    0.85 & 0.344 & 0.584 & 0.790 & 0.729 & 0.743 & 1.326 & 0.397 \\
    0.90 & 0.326 & 0.599 & 0.779 & 0.726 & 0.742 & 1.371 & 0.397 \\
    0.95 & 0.309 & 0.621 & 0.766 & 0.719 & 0.742 & 1.416 & 0.397 \\
    1.00 & 0.292 & 0.650 & 0.751 & 0.710 & 0.695 & 1.461 & 0.285 \\
    \bottomrule
  \end{tabular}
  \endgroup
\end{table}

\begin{table}[!htbp]
  \centering
  \caption{\textbf{Complete science-core weight sensitivity with Manuscript-Persistent.} All entries are unnormalized numerical results. Arrows indicate the preferred direction.}
  \label{tab:scicore_weight_persistent}
  \begingroup
  \scriptsize
  \setlength{\tabcolsep}{5pt}
  \renewcommand{\arraystretch}{1.02}
  \begin{tabular}{@{}c|cc|ccc|cc@{}}
    \toprule
    \multirow{2}{*}{$\boldsymbol{\alpha}$} &
    \multicolumn{2}{c}{\textbf{Within-paper stability}} &
    \multicolumn{3}{c}{\textbf{Joint stability-discrimination}} &
    \multicolumn{2}{c}{\textbf{Human alignment}} \\
    \cmidrule(lr){2-3}\cmidrule(lr){4-6}\cmidrule(lr){7-8}
    & MAD $\downarrow$ & Drift SD $\downarrow$ & ICC $\uparrow$ &
    SPR $\uparrow$ & Discrim. $\uparrow$ & H-MAE $\downarrow$ & Spearman $\uparrow$ \\
    \midrule
    0.00 & 0.476 & 0.875 & 0.697 & 0.586 & 0.682 & 1.261 & 0.423 \\
    0.05 & 0.461 & 0.830 & 0.712 & 0.601 & 0.727 & 1.269 & 0.408 \\
    0.10 & 0.447 & 0.787 & 0.726 & 0.617 & 0.727 & 1.276 & 0.408 \\
    0.15 & 0.432 & 0.746 & 0.741 & 0.634 & 0.727 & 1.284 & 0.408 \\
    0.20 & 0.418 & 0.707 & 0.755 & 0.652 & 0.729 & 1.291 & 0.408 \\
    0.25 & 0.403 & 0.670 & 0.769 & 0.669 & 0.732 & 1.299 & 0.408 \\
    0.30 & 0.389 & 0.636 & 0.782 & 0.687 & 0.736 & 1.306 & 0.408 \\
    0.35 & 0.374 & 0.604 & 0.794 & 0.704 & 0.736 & 1.314 & 0.408 \\
    0.40 & 0.360 & 0.577 & 0.804 & 0.720 & 0.738 & 1.321 & 0.408 \\
    0.45 & 0.347 & 0.554 & 0.812 & 0.734 & 0.739 & 1.329 & 0.408 \\
    0.50 & 0.333 & 0.535 & 0.819 & 0.747 & 0.738 & 1.336 & 0.402 \\
    0.55 & 0.329 & 0.522 & 0.823 & 0.756 & 0.752 & 1.349 & 0.378 \\
    0.60 & 0.325 & 0.514 & 0.825 & 0.763 & 0.751 & 1.361 & 0.378 \\
    0.65 & 0.321 & 0.512 & 0.824 & 0.766 & 0.751 & 1.374 & 0.378 \\
    0.70 & 0.316 & 0.517 & 0.820 & 0.766 & 0.754 & 1.386 & 0.378 \\
    0.75 & 0.312 & 0.527 & 0.814 & 0.762 & 0.750 & 1.399 & 0.378 \\
    0.80 & 0.308 & 0.542 & 0.806 & 0.756 & 0.747 & 1.411 & 0.378 \\
    0.85 & 0.304 & 0.563 & 0.795 & 0.747 & 0.745 & 1.424 & 0.378 \\
    0.90 & 0.300 & 0.588 & 0.782 & 0.736 & 0.744 & 1.436 & 0.378 \\
    0.95 & 0.296 & 0.618 & 0.767 & 0.723 & 0.744 & 1.449 & 0.378 \\
    1.00 & 0.292 & 0.650 & 0.751 & 0.710 & 0.695 & 1.461 & 0.285 \\
    \bottomrule
  \end{tabular}
  \endgroup
\end{table}

\clearpage
\section{Compute, Cost, and Efficiency}
\label{app:compute_cost}

We report API expenditure by experimental task. \oursb{} construction covers
the full-manuscript rewrites and reviewer-guided feedback used to construct the
controlled corpus. Review-only covers the general-purpose reviewer grid and
the repeated-review audit. The \oursm{} task includes science-core extraction,
the branch-policy evaluations, and embedding; its manuscript branch reuses the GPT-5.5 \textsc{Strict} reviews counted under Review-only, and score fusion introduces no additional API call. ReconstructReview includes
manuscript reconstruction and final review while reusing cached science cores.
The AI Scientist and OpenJudge entries cover their released agentic review
workflows. The task-level costs in Table~\ref{tab:execution_accounting} are the
recorded expenditures from the provider API consoles and include billable
retries and failed attempts.

\begin{table}[!t]
\centering
\caption{\textbf{API expenditure by experimental task.} Review-only costs are
aggregated by billing route rather than by model. Costs are the recorded
expenditures from the provider API consoles.}
\label{tab:execution_accounting}
\begingroup
\scriptsize
\setlength{\tabcolsep}{8pt}
\renewcommand{\arraystretch}{1.05}
\begin{tabular}{@{}lr@{}}
\toprule
Task & Cost (USD) \\
\midrule
\oursb{} construction & \$6,688.24 \\
Review-only (OpenAI API) & \$3,767.59 \\
Review-only (OpenRouter API) & \$870.60 \\
\oursm evaluation & \$3,043.48 \\
ReconstructReview & \$8,193.82 \\
AI Scientist & \$391.55 \\
OpenJudge & \$620.24 \\
\midrule
\textbf{Total API expenditure} & \textbf{\$23,575.52} \\
\bottomrule
\end{tabular}
\endgroup
\end{table}

OpenReviewer, CycleReviewer, DeepReviewer, and the trained ProReviewer backbone
are excluded from the API expenditure total because they were served locally
on a server equipped with 8 NVIDIA A100 GPUs.

\end{document}